\documentclass[11pt, a4paper]{article}
\usepackage{amsthm, amssymb, amsmath}
\usepackage{verbatim, float}
\usepackage{mathrsfs}
\usepackage{graphics}
\usepackage{subfig}
\usepackage[usenames,dvipsnames]{pstricks}
\usepackage{epsfig}
\usepackage{pst-grad} 
\usepackage{pst-plot} 
\usepackage{cases}
\usepackage{multirow}
\usepackage[top=2.4cm, bottom=2.4cm, left=2.5cm, right=2.5cm]{geometry}
\usepackage{url}
\usepackage{array}
\usepackage{tabu}
\usepackage{hyperref}

\newcommand{\nc}{\newcommand}

\theoremstyle{definition}
 
\newtheorem{theorem}{Theorem}[section]

\newtheorem{definition}[theorem]{Definition}
\newtheorem{lemma}[theorem]{Lemma}

\theoremstyle{remark}

\nc\remove[1]{}

\nc\bfa{{\boldsymbol a}}\nc\bfA{{\mathbf A}}\nc\cA{{\mathcal A}}
\nc\bfb{{\boldsymbol b}}\nc\bfB{{\mathbf B}}\nc\cB{{\mathcal B}}
\nc\bfc{{\boldsymbol c}}\nc\bfC{{\mathbf C}}\nc\cC{{\mathcal C}}
\nc\bfd{{\boldsymbol d}}\nc\bfD{{\mathbf D}}\nc\cD{{\mathcal D}}\nc\sD{{\mathscr D}}
\nc\bfe{{\boldsymbol e}}\nc\bfE{{\mathbf E}}\nc\cE{{\EuScript E}}
\nc\bff{{\boldsymbol f}}\nc\bfF{{\mathbf F}}\nc\cF{{\mathcal F}}
\nc\bfg{{\boldsymbol g}}\nc\bfG{{\mathbf G}}\nc\cG{{\mathcal G}}
\nc\bfh{{\boldsymbol h}}\nc\bfH{{\mathbf H}}\nc\cH{{\mathcal H}}
\nc\bfi{{\boldsymbol i}}\nc\bfI{{\mathbf I}}\nc\cI{{\mathcal I}}
\nc\bfj{{\boldsymbol j}}\nc\bfJ{{\mathbf J}}\nc\cJ{{\mathcal J}}
\nc\bfk{{\boldsymbol k}}\nc\bfK{{\mathbf K}}\nc\cK{{\mathcal K}}
\nc\bfl{{\boldsymbol l}}\nc\bfL{{\mathbf L}}\nc\cL{{\mathcal L}}\nc\sL{{\mathscr L}}
\nc\bfm{{\boldsymbol m}}\nc\bfM{{\mathbf M}}\nc\cM{{\mathcal M}}
\nc\bfn{{\boldsymbol n}}\nc\bfN{{\mathbf N}}\nc\cN{{\mathcal N}}
\nc\bfo{{\boldsymbol o}}\nc\bfO{{\mathbf O}}\nc\cO{{\mathcal O}}
\nc\bfp{{\boldsymbol p}}\nc\bfP{{\mathbf P}}\nc\cP{{\mathcal P}}
\nc\bfq{{\boldsymbol q}}\nc\bfQ{{\mathbf Q}}\nc\cQ{{\mathcal Q}}\nc\sQ{{\mathscr Q}}
\nc\bfr{{\boldsymbol r}}\nc\bfR{{\mathbf R}}\nc\cR{{\mathcal R}}
\nc\bfs{{\boldsymbol s}}\nc\bfS{{\mathbf S}}\nc\cS{{\mathcal S}}
\nc\bft{{\boldsymbol t}}\nc\bfT{{\mathbf T}}\nc\cT{{\mathcal T}}\nc\sT{{\mathscr T}}
\nc\bfu{{\boldsymbol u}}\nc\bfU{{\mathbf U}}\nc\cU{{\mathcal U}}
\nc\bfv{{\boldsymbol v}}\nc\bfV{{\mathbf V}}\nc\cV{{\mathcal V}}
\nc\bfw{{\boldsymbol w}}\nc\bfW{{\mathbf W}}\nc\cW{{\mathcal W}}\nc\sW{{\mathscr W}}
\nc\bfx{{\boldsymbol x}}\nc\bfX{{\mathbf Z}}\nc\cX{{\EuScript X}}
\nc\bfy{{\boldsymbol y}}\nc\bfY{{\mathbf Y}}\nc\cY{{\EuScript Y}}\nc\sY{{\mathscr Y}}
\nc\bfz{{\boldsymbol z}}\nc\bfZ{{\mathbf Z}}\nc\cZ{{\mathcal Z}}\nc\sZ{{\mathscr Z}}

\def\h_q{\qopname\relax{no}{h_q}}
\def\h{\qopname\relax{no}{h}}

\def\sgn{\qopname\relax{no}{sgn}}

\begin{document}
%
%
%
%

\title{Submodular Hypergraphs: $p$-Laplacians, Cheeger Inequalities and Spectral Clustering}
\author{ Pan~Li \quad Olgica Milenkovic}
\footnotetext[1]{The authors are with the Coordinated Science Laboratory, Department of Electrical and Computer Engineering, University of Illinois at Urbana-Champaign (email: panli2@illinois.edu, milenkov@illinois.edu)}
\footnotetext[2]{The code for Algorithm 2 (inverse power method) can be found in https://github.com/lipan00123/IPM-for-submodular-hypergraphs.}
\date{}
\maketitle
\renewcommand*{\thefootnote}{\arabic{footnote}}

\begin{abstract}
We introduce submodular hypergraphs, a family of hypergraphs that have different submodular weights associated with different cuts of hyperedges. Submodular hypergraphs arise in clustering applications in which higher-order structures carry relevant information. For such hypergraphs, we define the notion of $p$-Laplacians and derive corresponding nodal domain theorems and $k$-way Cheeger inequalities. We conclude with the description of algorithms for computing the spectra of $1$- and $2$-Laplacians that constitute the basis of new spectral hypergraph clustering methods. 
\end{abstract}

\section{Introduction}
Spectral clustering algorithms are designed to solve a relaxation of the graph cut problem based on graph Laplacians that capture pairwise dependencies between vertices, and produce sets with small conductance that represent clusters. Due to their scalability and provable performance guarantees, spectral methods represent one of the most prevalent graph clustering approaches~\cite{chung1997spectral,ng2002spectral}. 

Many relevant problems in clustering, semisupervised learning and MAP inference~\cite{zhou2007learning,hein2013total,pmlr-v70-zhang17d} involve higher-order vertex dependencies that require one to consider hypergraphs instead of graphs. To address spectral hypergraph clustering problems, several approaches have been proposed that typically operate by first projecting the hypergraph onto a graph via \emph{clique expansion} and then performing spectral clustering on graphs~\cite{zhou2007learning}. Clique expansion involves transforming a weighted hyperedge into a weighted clique such that the graph cut weights approximately preserve the cut weights of the hyperedge. Almost exclusively, these approximations have been based on the assumption that each hyperedge cut has the same weight, in which case the underlying hypergraph is termed \emph{homogeneous}.

However, in image segmentation, MAP inference on Markov random fields~\cite{arora2012generic, shanu2016min}, network motif studies~\cite{li2017inhomogeneous,benson2016higher,tsourakakis2017scalable} and rank learning~\cite{li2017inhomogeneous}, higher order relations between vertices captured by hypergraphs are typically associated with different cut weights. In~\cite{li2017inhomogeneous}, Li and Milenkovic generalized the notion of hyperedge cut weights by assuming that different hyperedge cuts have different weights, and that consequently, each hyperedge is associated with a vector of weights rather than a single scalar weight. If the weights of the hyperedge cuts are submodular, then one can use a graph with nonnegative edge weights to efficiently approximate the hypergraph, provided that the largest size of a hyperedge is a relatively small constant. This property of the projected hypergraphs allows one to leverage spectral hypergraph clustering algorithms based on clique expansions with provable performance guarantees. Unfortunately, the clique expansion method in general has two drawbacks: The spectral clustering algorithm for graphs used in the second step is merely quadratically optimal, while the projection step can cause a large distortion. 

To address the quadratic optimality issue in graph clustering, Amghibech~\cite{amghibech2003eigenvalues} introduced the notion of $p$-Laplacians of graphs and derived Cheeger-type inequalities for the second smallest eigenvalue of a $p$-Laplacian, $p>1$, of a graph. These results motivated B$\ddot{\text{u}}$hler and Hein's work~\cite{buhler2009spectral} on spectral clustering based on $p$-Laplacians that provided tighter approximations of the Cheeger constant. Szlam and Bresson~\cite{szlam2010total} showed that the $1$-Laplacian allows one to exactly compute the Cheeger constant, but at the cost of computational hardness~\cite{chang2016spectrum}. Very little is known about the use of $p$-Laplacians for hypergraph clustering and their spectral properties. 

To address the clique expansion problem, Hein et al.~\cite{hein2013total} introduced a clustering method for homogeneous hypergraphs that avoids expansions and works directly with the total variation of homogeneous hypergraphs, without investigating the spectral properties of the operator. The only other line of work trying to mitigate the projection problem is due to Louis~\cite{louis2015hypergraph}, who used a natural extension of $2$-Laplacians for homogeneous hypergraphs, derived quadratically-optimal Cheeger-type inequalities and proposed a semidefinite programing (SDP) based algorithm whose complexity scales with the size of the largest hyperedge in the hypergraph. 
 
Our contributions are threefold. First, we introduce submodular hypergraphs. Submodular hypergraphs allow one to perform hyperedge partitionings that depend on the subsets of elements involved in each part, thereby respecting higher-order and other constraints in graphs (see~\cite{li2017inhomogeneous,  arora2012generic, fix2013structured} for applications in food network analysis, learning to rank, subspace clustering and image segmentation). Second, we define $p$-Laplacians for submodular hypergraphs and generalize the corresponding discrete nodal domain theorems~\cite{tudisco2016nodal,chang2017nodal} and higher-order Cheeger inequalities. Even for homogeneous hypergraphs, nodal domain theorems were not known and only one low-order Cheeger inequality for $2$-Laplacians was established by Louis~\cite{louis2015hypergraph}. An analytical obstacle in the development of such a theory is the fact that $p$-Laplacians of hypergraphs are operators that act on vectors and produce \emph{sets of values}. Consequently, operators and eigenvalues have to be defined in a set-theoretic manner. Third, based on the newly established spectral hypergraph theory, we propose two spectral clustering methods that learn the second smallest eigenvalues of $2$- and $1$-Laplacians. The algorithm for $2$-Laplacian eigenvalue computation is based on an SDP framework and can provably achieve quadratic optimality with an $O(\sqrt{\zeta(E)})$ approximation constant, where $\zeta(E)$ denotes the size of the largest hyperedge in the hypergraph. The algorithm for $1$-Laplacian eigenvalue computation is based on the inverse power method (IPM)~\cite{hein2010inverse} that only has convergence guarantees. The key novelty of the IPM-based method is that the critical inner-loop optimization problem of the IPM is efficiently solved by algorithms recently developed for decomposable submodular minimization~\cite{jegelka2013reflection, ene2015random, li2018revisiting}. Although without performance guarantees, given that the $1$-Laplacian provides the tightest approximation guarantees, the IPM-based algorithm -- as opposed to the clique expansion method~\cite{li2017inhomogeneous} -- performs very well empirically even when the size of the hyperedges is large. This fact is illustrated on several UC Irvine machine learning datasets available from~\cite{asuncion2007uci}.

\subsection{Other related works and applications}
In this work, we concentrated on rigorously characterizing the spectra of $p$-Laplacians of submodular hypergraphs. The obtained Cheeger inequalities and nodal domain theorems provide strong implication for spectral clustering. However, one should note that the Laplacians of submodular hypergraphs potentially hold a much wider range of usage. First, the spectrum of graph Laplacian can be used to construct the wavelet and Fourier frames of non-Euclidean topology. These frames have been widely leveraged in graph signal processing tasks~\cite{hammond2011wavelets, shuman2013emerging,zou2015nonparametric}. It is interesting to investigate whether the proposed submodular structures may coincide with some topology and induce other applicable frames for signal processing. Second, from Laplacian Eigenmap~\cite{belkin2003laplacian} to the graph convolutional neural network~\cite{defferrard2016convolutional}, graph Laplacians are the fundamental tools of network embedding that is widely used for many downstream machine learning jobs~\cite{tang2015line}. The proposed Laplacians potentially lead to a new class of graph embedding methods that may better capture complex high-order structures of  networks via the submodular assumption. Some new follow-up works along this line can be found in~\cite{yang2018meta, yadati2018hypergcn}. Also, our submodular hypergraphs consist of an important subclass of a more general concept ``submodularity over edges''. The latter imposes submodularity on the set of edges instead of vertices. Although many combinatorial optimization problems has been looked into in this more general settings~\cite{jegelka2017graph,mitrovic2018submodularity}, it is not clear whether it has a Laplacian formulation and some underlying spectral theory can be derived. 

The authors also would like to refer some follow-up works on efficient algorithms to compute the min-cut and PageRank for submodular hypergraphs, which essentially correspond to the decomposable submodular function minimization problems (DSFM) with incidence relations~\cite{li2018revisiting} and quadratic DSFM problems~\cite{li2018quadratic}.  

In an independent work, Yoshida also considered the spectral graph theory related to submodular hypergraphs~\cite{yoshida2017cheeger}. That work focused on 2-Laplacian and its first non-trivial eigenpair, and thus did not establish the most general $k$-way Cheeger inequalities and discrete nodal domain theorems shown in this work. 

The paper is organized as follows. Section~\ref{sec:preliminaries} contains an overview of graph Laplacians and introduces the notion of submodular hypergraphs. The section also contains a description of hypergraph Laplacians, and relevant concepts in submodular function theory. Section~\ref{sec:plap} presents the fundamental results in the spectral theory of $p$-Laplacians, while Section~\ref{sec:algs} introduces two algorithms for evaluating the second largest eigenvalue of $p$-Laplacians needed for $2$-way clustering. Section~\ref{sec:data} presents experimental results. All proofs are relegated to the Appendix.

\section{Mathematical Preliminaries} \label{sec:preliminaries}

A weighted graph $G=(V,E,w)$ is an ordered pair of two sets, the vertex set $V=[N]=\{{1,2,\ldots,N\}}$ and the edge set $E \subseteq V \times V$, equipped with a weight function $w: E \to \mathbb{R}^{+}$. 

A cut $C=(S,\bar{S})$ is a bipartition of the set $V$, while the cut-set (boundary) of the cut $C$ is defined as the set of edges that have one endpoint in $S$ and one in the complement of $S$, $\bar{S}$, i.e., 
$\partial S={\displaystyle \{(u,v)\in E\mid u\in S,v\in \bar{S}\}}$. The weight of the cut induced by $S$ equals $\text{vol}(\partial S)=\sum_{u \in S,\,v\in \bar{S}}\, w_{uv}$, while the conductance of the cut is defined as
$$c(S)=\frac{\text{vol}(\partial S)}{\min\{{\text{vol}(S), \text{vol}(\bar{S})\}}},$$
where $\text{vol}(S)=\sum_{u \in S}\, \mu_u$, and $\mu_u=\sum_{v \in V} w_{uv}$. Whenever clear from the context, for $e=(uv)$, we write $w_e$ instead of $w_{uv}$. Note that in this setting, the vertex weight values $\mu_u$ are determined based on the weights of edges $w_e$ incident to $u$. Clearly, one can use a different choice for these weights and make them independent from the edge weights, which is a generalization we pursue in the context of submodular hypergraphs. The smallest conductance of any bipartition of a graph $G$ is denoted by $h_2$ and referred to as the Cheeger constant of the graph. 

A generalization of the Cheeger constant is the \emph{$k-$way Cheeger constant} of a graph $G$. Let $P_k$ denote the set of all partitions of $V$ into $k$-disjoint nonempty subsets, i.e.,
$P_k=\{(S_1, S_2, ..., S_k)| S_i \subset V, S_i \neq \emptyset, S_i \cap S_j = \emptyset, \forall i,j \in [k], i\neq j\}$. The $k-$way Cheeger constant is defined as
\begin{align*}
h_k=\min_{(S_1, S_2, ..., S_k)\in P_k} \max_{i \in[k]} \, c(S_i).
\end{align*}

Spectral graph theory provides a means for bounding the Cheeger constant using the (normalized) Laplacian matrix of the graph, defined as $L=D-A$ and $L=I-D^{-1/2}AD^{-1/2}$, respectively. Here, $A$ stands for the adjacency matrix of the graph, $D$ denotes the diagonal degree matrix, while $I$ stands for the identity matrix. The graph Laplacian is an operator $\triangle_2^{(g)}$~\cite{chung1997spectral} that satisfies
\begin{align*}
\langle x, \triangle_2^{(g)}(x)\rangle = \sum_{(uv) \in E} w_{uv} (x_u - x_v)^2. 
\end{align*}
A generalization of the above operator termed the $p$-Laplacian operator of a graph $\triangle_p^{(g)}$ was introduced by Amghibech in~\cite{amghibech2003eigenvalues}, where
\begin{align*}
\langle x, \triangle_p^{(g)} (x)\rangle = \sum_{(uv) \in E}  w_{uv}  |x_u - x_v|^p.
\end{align*}

The well known Cheeger inequality asserts the following relationship between $h_2$ and $\lambda$, the second smallest eigenvalue of the normalized Laplacian $\triangle_2^{(g)}$ of a graph: 
\begin{equation} \label{cheeger}
h_2 \leq \sqrt{2\lambda} \leq 2\sqrt{h}_2. \notag
\end{equation}
It can be shown that the cut $\hat{h}_2$ dictated by the elements of the eigenvector associated with $\lambda$ satisfies $\hat{h}_2 \leq  \sqrt{2\lambda}$, which implies $\hat{h}_2 \leq 2\sqrt{h}_2$. Hence, spectral clustering provides a quadratically optimal graph partition. 

\subsection{Submodular Hypergraphs}

A weighted hypergraph $G=(V,E,w)$ is an ordered pair of two sets, the vertex set $V=[N]$ and the hyperedge set $E \subseteq 2^V$, equipped with a weight function $w: E \to \mathbb{R}^{+}$. The relevant notions of cuts, boundaries and volumes for hypergraphs can be defined in a similar manner as for graphs. If each cut of a hyperedge $e$ has the same weight $w_e$, we refer to the cut as a homogeneous cut and the corresponding hypergraph as a homogeneous hypergraph. 

For a ground set $\Omega$, a set function $f: 2^{\Omega} \to \mathbb{R}$ is termed submodular if for all $S,T \subseteq \Omega$, one has $f(S)+f(T) \geq f(S \cup T)+f(S \cap T)$. 

A weighted hypergraph $G=(V,E,\boldsymbol{\mu},\mathbf{w})$ is termed a \emph{submodular hypergraph} with vertex set $V$, hyperedge set $E$ and positive vertex weight vector $\boldsymbol{\mu} \triangleq \{\mu_v\}_{v\in V},$ if each hyperedge $e\in E$ is associated with a submodular weight function $w_e(\cdot): 2^e\rightarrow [0, 1]$. In addition, we require the weight function $w_e(\cdot)$ to be:

1) Normalized, so that $w_e(\emptyset)= 0$, and all cut weights corresponding to a hyperedge $e$ are normalized by $\vartheta_e= \max_{S\subseteq e} w_e(S)$. In this case, $w_e(\cdot) \in [0,1]$; 

2) Symmetric, so that $w_e(S)=w_e(e\backslash S)$ for any $S\subseteq e$;

The submodular hyperedge weight functions are summarized in the vector  $\mathbf{w} \triangleq \{(w_e, \vartheta_e)\}_{e\in E}.$
If $w_e(S)=1$ for all $S\in 2^e\backslash\{\emptyset, e\}$, submodular hypergraphs reduce to homogeneous hypergraphs. We omit the designation homogeneous whenever there is no context ambiguity. 

Clearly, a vertex $v$ is in $e$ if and only if $w_e(\{v\}) > 0$: If $w_e(\{v\}) = 0$, the submodularity property implies that $v$ is not \emph{incident} to $e$, as for any $S\subseteq e\backslash\{v\}$, $|w_e(S\cup\{v\})-w_e(S)|\leq w_e(\{v\})=0$. 

We define the degree of a vertex $v$ as $d_v = \sum_{e\in E:\, v\in e} \vartheta_e$, i.e., as the sum of the max weights of edges incident to the vertex $v$. Furthermore, for any vector  $y\in\mathbb{R}^N$, we define the projection weight of $y$ onto any subset $S \subseteq V$ as $y(S) = \sum_{v\in S} y_v$. The volume of a subset of vertices $S\subseteq V$ equals $\text{vol}(S)=\sum_{v\in S} \mu_v.$

For any $S \subseteq V$, we generalize the notions of the boundary of $S$ and the volume of the boundary of $S$ according to $\partial S = \{e\in E | e\cap S\neq \emptyset,   e\cap \bar{S} \neq \emptyset\}$, and 
\begin{align}\label{cut} 
\text{vol}(\partial S) = \sum_{e\in \partial S} \vartheta_e w_{e}(S) = \sum_{e\in E} \vartheta_e w_{e}(S),
\end{align}
respectively. Then, the normalized cut induced by $S$, the Cheeger constant and the $k$-way Cheeger constant for hypergraphs are defined in an analogous manner as for graphs.

\subsection{Laplacian Operators for Hypergraphs}
We introduce next $p$-Laplacians of hypergraphs and a number of relevant notions associated with Laplacian operators.

Hein et al.\cite{hein2013total} connected $p$-Laplacians $\triangle_p^{(h)}$ for homogeneous hypergraphs with the total variation via
\begin{align*}
\langle x, \triangle_p^{(h)}(x)\rangle = \sum_{e\in E}  w_{e} \max_{u,v\in e}|x_u - x_v|^p,
\end{align*}
where $w_{e}$ denotes the weight of a homogeneous hyperedge $e$. They also introduced the Inverse Power Method (IPM) to evaluate the spectrum of the hypergraph $1$-Laplacian $\triangle_1^{(h)}$~\cite{hein2013total}, but did not establish any performance guarantees. In an independent line of work, Louis~\cite{louis2015hypergraph} introduced a quadratic variant of a hypergraph Laplacian
\begin{align*}
\langle x, \triangle_2^{(h)}(x)\rangle = \sum_{e\in E} w_{e}\max_{u,v\in e}(x_u - x_v)^2.
\end{align*}
He also derived a Cheeger-type inequality relating the second smallest eigenvalue $\lambda$ of $\triangle_2^{(h)}$ and the Cheeger constant of the hypergraph $h_2$ that reads as $\hat{h}_2\leq O(\sqrt{\log \zeta(E)})\sqrt{\lambda} \leq O(\sqrt{\log \zeta(E)})\sqrt{h_2}$. Compared to the result of graph~\eqref{cheeger}, for homogeneous hypergraphs, $\log\zeta(E)$ plays as some additional difficulty to approximate $h_2$. Learning the spectrum of generalizations of hypergraph Laplacians can be an even more challenging task.

\subsection{Relevant Background on Submodular Functions}

Given an arbitrary set function $F: 2^V \rightarrow \mathbb{R}$ satisfying $F(V)=0$, the \emph{Lov{\'a}sz extension}~\cite{lovasz1983submodular} $f:  \mathbb{R}^N \rightarrow  \mathbb{R}$ of $F$ is defined as follows: For any vector $x\in \mathbb{R}^N$, we order its entries in nonincreasing order $x_{i_1}\geq x_{i_2}\geq \cdots \geq x_{i_n}$ while breaking the ties arbitrarily, and set 
\begin{align}\label{lovazexp}
f(x) = \sum_{j=1}^{N-1}F(S_j)(x_{i_j}-x_{i_{j+1}}), 
\end{align}
with $S_{j} = \{i_1, i_2,...,i_j\}$. For submodular $F$, the Lov{\'a}sz extension is a convex function~\cite{lovasz1983submodular}.

Let $\mathbf{1}_S\in \mathbb{R}^N$ be the indicator vector of the set $S$. Hence, for any $S\subseteq V$, one has $F(S) = f(\mathbf{1}_S)$. For a submodular $F$, we define a convex set termed the \emph{base polytope} 
\begin{align*}
\mathcal{B} \triangleq \{y\in\mathbb{R}^N| y(S)\leq F(S),\;\text{for all }S\subseteq V,  \text{and such that } y(V)=F(V)=0\}.
\end{align*}
According to the defining property of submodular functions~\cite{lovasz1983submodular}, we may write $f(x) = \max_{y\in \mathcal{B}} \langle y, x\rangle$.
 
The subdifferential $\nabla f(x)$ of $f$ is defined as 
\begin{align*}
 \{y\in \mathbb{R}^N \, | \, f(x')-f(x)  \geq \langle y, x'-x \rangle, \, \forall x'\in \mathbb{R}^N\}.
\end{align*} 

An important result from~\cite{bach2013learning} characterizes the subdifferentials $\nabla f(x)$: If $f(x)$ is the Lov{\'a}sz extension of a submodular function $F$ with base polytope $\mathcal{B}$, then 
\begin{align}\label{subgradient}
\nabla f(x) =\arg\max_{y \in \mathcal{B}} \langle y, x \rangle.
\end{align}
Observe that $\nabla f(x)$ is a set and that the right hand side of the definition represents a set of maximizers of the objective function. If $f(x)$ is the Lov{\'a}sz extension of a submodular function, then $\langle q, x\rangle  = f(x)$ for all $q\in \nabla f(x)$. 

For each hyperedge $e\in E$ of a submodular hypergraph, following the above notations, we let $\mathcal{B}_e$, $\mathcal{E}(\mathcal{B}_e)$, $f_e$ denote the base polytope, the set of extreme points of the base polytope, and the Lov{\'a}sz extension of the submodular hyperedge weight function $w_e$, respectively. Note that for any $S\subseteq V$, $w_e(S) = w_e(S\cap e)$. 
Consequently, for any $y\in \mathcal{B}_e$, $y_v = 0$ for $v\not\in e$. Since $\nabla f_e \subseteq \mathcal{B}_e$, it also holds that $(\nabla f_e)_v = 0$ for $v\notin e$. When using formula~\eqref{lovazexp} to explicitly describe the Lov{\'a}sz extension $f_e$, we can either use a vector $x$ of dimension $N$ or only those of its components that lie in $e$. Furthermore, in the later case, $|\mathcal{E}(\mathcal{B}_e)|=|e|!$.

\section{$p$-Laplacians for submodular hypergraphs and the spectra}\label{sec:plap}
We start our discussion by defining the notion of a $p$-Laplacian operator for submodular hypergraphs. We find the following definitions useful for our subsequent exposition. 

Let $\text{sgn}(\cdot)$ be the sign function defined as $\text{sgn}(a)=1,$ for $a > 0$, $\text{sgn}(a)=-1,$ for $a <0$, and $\text{sgn}(a)=[-1,1],$ for $a = 0$. For all $v\in V$, define the entries of a vector $\varphi_p$ over $\mathbb{R}^N$ according to $(\varphi_p(x))_v=|x_v|^{p-1}\text{sgn}(x_v)$. Let $\|x\|_{\ell_p, \mu}= (\sum_{v\in V}\mu_v |x_{v}|^p)^{1/p}$ and $\mathcal{S}_{p,\mu} \triangleq \{x\in \mathbb{R}^N | \|x\|_{\ell_p, \mu}=1\}$. For a function $\Phi$ over $\mathbb{R}^N$, let $\Phi|_{\mathcal{S}_{p,\mu}}$ stand for $\Phi$ restricted to $\mathcal{S}_{p,\mu}$.

\begin{definition}\label{lapdef}
The $p$-Laplacian operator of a submodular hypergraph, denoted by $\triangle_p$ ($p\geq 1$), is defined for all $x\in \mathbb{R}^N$ according to
\begin{align} \label{lap1}
\langle x,\triangle_p(x)\rangle\triangleq Q_p(x) = \sum_{e\in E} \vartheta_e f_e(x)^p .
\end{align}
Hence, $\triangle_p(x)$ may also be specified directly as an operator over $\mathbb{R}^N$ that reads as
\begin{align*}
\triangle_p(x)& = \left\{ \begin{array}{cc} \sum_{e\in E} \vartheta_e f_e(x)^{p-1} \nabla f_e(x) & p>1, \\  \sum_{e\in E} \vartheta_e \nabla f_e(x) & p=1.  \end{array}\right. 
\end{align*}
\end{definition}
\begin{definition}\label{eigendef}
A pair $(\lambda, x)\in \mathbb{R}\times \mathbb{R}^N/\{\mathbf{0}\}$ is called an eigenpair of the $p$-Laplacian $\triangle_p$ if $\triangle_p(x) \cap \lambda U \, \varphi_p(x) \neq \emptyset$.
\end{definition}
As $f_e(\mathbf{1}) = 0$, we have $\triangle_p(\mathbf{1})=0$, so that $(0,\mathbf{1})$ is an eigenpair of the operator $\triangle_p$.
A $p$-Laplacian operates on vectors and produces sets. In addition, since for any $t>0$, $\triangle_p(tx) = t^{p-1}\triangle_p(x)$ and $\varphi_p(tx) = t^{p-1} \varphi_p(x)$, $(tx, \lambda)$ is an eigenpair if and only if $(x, \lambda)$ is an eigenpair. 
Hence, one only needs to consider normalized eigenpairs: In our setting, we choose eigenpairs that lie in $\mathcal{S}_{p,\mu}$ for a suitable choice for the dimension of the space. 

For linear operators, the Rayleigh-Ritz method~\cite{gould1966variational} allows for determining approximate solutions to eigenproblems and provides a variational characterization of eigenpairs based on the critical points of functionals. To generalize the method, we introduce two even functions,
\begin{align*}
 \tilde{Q}_p(x)\triangleq Q_p(x)|_{\mathcal{S}_{p,\mu} }, \quad R_p(x) \triangleq \frac{Q_p(x)}{\|x\|_{\ell_p, \mu}^p}.
 \end{align*}
 \begin{definition}
A point $x \in \mathcal{S}_{p,\mu}$ is termed a \emph{critical point} of $R_p(x)$ if $0\in \nabla R_p(x)$. Correspondingly, $R_p(x)$ is termed a \emph{critical value} of $R_p(x)$. Similarly, $x$ is termed a \emph{critical point} of $\tilde{Q}_p$ if there exists a $\sigma \in \nabla Q_p(x)$ such that $P(x) \sigma = 0$, where $P(x) \sigma$ stands for the projection of $\sigma$ onto the tangent space of $\mathcal{S}_{p,\mu}$ at the point $x$. Correspondingly, $\tilde{Q}_p(x)$ is termed a \emph{critical value} of $\tilde{Q}_p$. 
\end{definition}
The relationships between the critical points of $\tilde{Q}_p(x)$ and $R_p(x)$ and the eigenpairs of $\triangle_p$ relevant to our subsequent derivations are listed in Theorem~\ref{eigenpairthm}. 
\begin{theorem}\label{eigenpairthm}
A pair $(\lambda, x)$ ($x\in \mathcal{S}_{p,\mu}$) is an eigenpair of the operator $\triangle_p$\\
1) if and only if $x$ is a critical point of $\tilde{Q}_p$ with critical value $\lambda$, and provided that $p\geq 1$. \\
2) if and only if $x$ is a critical point of $R_p$ with critical value $\lambda$, and provided that $p>1$. \\
3) if $x$ is a critical point of $R_p$ with critical value $\lambda$, and provided that $p=1$. 
\end{theorem}
The critical points of $\tilde{Q}_p$ bijectively characterize eigenpairs for all choices of $p\geq 1$. However, $R_p$ has the same property only if $p>1$. This is a consequence of the nonsmoothness of the set $\mathcal{S}_{1,\mu}$, which has been observed for graphs as well (See the examples in Section 2.2 in~\cite{chang2016spectrum}).

Once Theorem~\ref{eigenpairthm} has been established, a standard way to analyze the spectrum of $\triangle_p$ is to study the critical points of $\tilde{Q}_p = Q_p(x)|_{\mathcal{S}_{p,\mu}}$. A crucial component within this framework is the Lusternik-Schnirelman theory that allows to characterize a series of these critical points. As $Q_p$ and $\mathcal{S}_{p,\mu}$ are symmetric, one needs to use the notion of a Krasnoselski genus, defined below. This type of approach has also been used to study the spectrum of $p$-Laplacians of graphs, and the readers interested in the mathematical theory behind the derivations are referred to~\cite{chang2016spectrum, tudisco2016nodal} and references therein for more details. 
 
 \begin{definition}
Let $A\subset \mathbb{R}^N/\{0\}$ be a closed and symmetric set. The Krasnoselski genus of $A$ is defined as
\begin{equation}
\mathcal{\gamma}(A)=\left\{ 
\begin{array}{l}
0 ,\quad \text{if}\;A=\emptyset, \\
\inf \{k\in \mathbb{Z}^+| \exists\;\text{odd continuous}\;h:\;A\rightarrow \mathbb{R}^k\backslash \{0\}\} \\
\infty \quad \text{if for any finite $k\in \mathbb{Z}^+$, no such $h$ exists.}
\end{array}
\right.
\end{equation}
 \end{definition}
We now focus on a particular subset of $\mathcal{S}_{p,\mu}$, defined as $$\mathcal{F}_k(\mathcal{S}_{p,\mu})\triangleq\{A \subseteq \mathcal{S}_{p,\mu}| A=-A, \text{closed}, \gamma(A)\geq k\}.$$ As $Q_p$ may not be differentiable, we apply Chang's generalization of the Lusternik-Schnirelman theorem for \emph{locally Lipschitz continuous} functionals defined on smooth Banach-Finsler manifolds (corresponding to the case $p>1$) and those defined on piecewise linear manifolds (corresponding to the case $p=1$). 

\begin{definition}
We say $g: \mathcal{S}_{p,\mu}\rightarrow \mathbb{R}$ is \emph{locally Lipschitz}: if for each $x\in \mathcal{S}_{p,\mu}$, there exists a neighborhood $\mathcal{N}_x$ of $x$ and a constant $C$ depending on $\mathcal{N}_x$ such that $|g(x') - g(x)|\leq C \|x'-x\|_{\ell_2}$ for any $x'\in \mathcal{S}_{p,\mu} \cap \mathcal{N}_x$.
\end{definition}

\begin{theorem}[Theorem 3.2~\cite{chang1981variational} and Theorem 4.9~\cite{chang2016spectrum}]
Suppose function $g: \mathcal{S}_{p,\mu}\rightarrow \mathbb{R}$ is \emph{locally Lipschitz}, even, bounded below, then
\begin{align*}
\min_{A: \mathcal{F}_k(\mathcal{S}_{p,\mu})}\max_{x\in A} g(x) \quad \quad k =1,2,...,N
\end{align*}
characterize the critical values of $g$. 
\end{theorem}
It is easy to check $\tilde{Q}_p$ is locally Lipschitz, even and bounded below. By invoking the Lusternik-Schnirelman theorem, we claim that there are at least $n$ critical values of $\tilde{Q}_p$ equaling
\begin{align}\label{defcrit}
\lambda_k^{(p)} = \min_{A: \mathcal{F}_k(\mathcal{S}_{p,\mu})}\max_{x\in A} \tilde{Q}_p, \quad \quad k =1,2,...,N.
\end{align}
Note that as $\mathcal{F}_{k+1}(\mathcal{S}_{p,\mu}) \subseteq \mathcal{F}_{k}(\mathcal{S}_{p,\mu})$, $\lambda_{k+1}^{(p)} \geq \lambda_{k}^{(p)}$. 
Combining~\eqref{defcrit} and Theorem~\ref{eigenpairthm}, $\{\lambda_k^{(p)}\}_{k\in[N]}$ are a collection of eigenvalues of p-Laplacian operators $\triangle_p$.

\subsection{Discrete Nodal Domain Theorem for $p-$Laplacians}

Nodal domain theorems are essential for understanding the structure of eigenvectors of operators and they have been the subject of intense study in geometry and graph theory alike~\cite{biyikoglu2007laplacian}. The eigenfunctions of a Laplacian operator may take positive and negative values. The signs of the values induce a partition of the vertices in $V$ into maximal connected components on which the sign of the eigenfunction does not change: These components represent the nodal domains of the eigenfunction and approximate the clusters of the graphs.

Davies et al.~\cite{briandavies2001discrete} derived the first discrete nodal domain theorem for the $\triangle_2^{(g)}$ operator. Chang et al.~\cite{chang2017nodal} and Tudisco et al.~\cite{tudisco2016nodal} generalized these theorem for $\triangle_1^{(g)}$ and $\triangle_p^{(g)}$ ($p>1$) of graphs. In what follows, we prove that  the discrete nodal domain theorem applies to $\triangle_p$ of submodular hypergraphs. 

As every nodal domain theorem depends on some underlying notion of connectivity, we first define the relevant notion of connectivity for submodular hypergraphs. In a graph or a homogeneous hypergraph, vertices on the same edge or hyperedge are considered to be connected. However, this property does not generalize to submodular hypergraphs, as one can merge two nonoverlapping hyperedges into one without changing the connectivity of the hyperedges. To see why this is the case, consider two hyperedges $e_1$ and $e_2$ that are nonintersecting. One may transform the submodular hypergraph so that it includes a  hyperedge $e=e_1\cup e_2$ with weight $w_e = w_{e_1}+w_{e_2}$. This transformation essentially does not change the submodular hypergraph, but in the newly obtained hypergraph, according to the standard definition of connectivity, the vertices in $e_1$ and $e_2$ are connected. This problem may be avoided by defining connectivity based on the volume of the boundary set.  
\begin{definition}
Two distinct vertices $u,v \in V$ are said to be \emph{connected} if for any $S$ such that $u\in S$ and $v \notin S$, $\text{vol}(\partial S)> 0$. A submodular hypergraph is \emph{connected} if for any non-empty $S\subset V$, one has $\text{vol}(\partial S)> 0$.
\end{definition}
According to the following lemma, it is always possible to transform the weight functions of submodular hypergraph in such a way as to preserve connectivity.  
\begin{lemma} \label{reduce}
Any submodular hypergraph $G=(V,E, \mathbf{w}, \boldsymbol{\mu})$ can be reduced to another submodular hypergraph $G'=(V, E', \mathbf{w}', \boldsymbol{\mu})$ without changing $\text{vol}(\partial S)$ for any $S\subseteq V$ and ensuring that for any $e\in E'$, and $u,v\in e$, $u$ and $v$ are connected. 
\end{lemma}
\begin{definition}
Let $x\in \mathbb{R}^N$. A \emph{positive (respectively, negative) strong nodal domain} is the set of vertices of a maximally connected induced subgraph of $G$ such that $\{v\in V | x_v> 0\}$ (respectively, $\{v\in V | x_v<0\}$). A \emph{positive (respectively, negative) weak nodal domain} is defined in the same manner, except for changing the strict inequalities as $\{v\in V | x_v\geq 0\}$ (\emph{respectively}, $\{v\in V | x_v\leq 0\}$). 
\end{definition}
The following lemma establishes that for a connected submodular hypergraph $G$, all nonconstant eigenvectors of the operator $\triangle_p$ correspond to nonzero eigenvalues. 
\begin{lemma}\label{nontrivial}
If $G$ is connected, then all eigenvectors associated with the zero eigenvalue have constant entries. 
\end{lemma}
We next state new nodal domain theorems for submodular hypergraph $p-$Laplacians. The results imply the bounds for the numbers of nodal domains induced from eigenvectors of $p$-Laplacian do not essentially change compared to those for graphs~\cite{tudisco2016nodal}.
We do not consider the case $p=1$, although it is possible to adapt the methods for analyzing the $\triangle_1^{(g)}$ operators of graphs to $\triangle_1$ operators of submodular hypergraphs. Such a generalization requires extensions of the critical-point theory to piecewise linear manifolds~\cite{chang2016spectrum}.  
\begin{theorem} \label{nodal}
Let $p>1$ and assume that $G$ is a connected submodular hypergraph. Furthermore, let the eigenvalues of $\triangle_p$ be ordered as $0=\lambda_1^{(p)}<\lambda_2^{(p)}\leq\cdots\leq \lambda_{k-1}^{(p)}<\lambda_k^{(p)} =\cdots=\lambda_{k+r-1}^{(p)}<\lambda_{k+r}^{(p)}\leq \cdots \leq \lambda_n^{(p)}$, with $\lambda_k^{(p)}$ having multiplicity $r$. Let $x$ be an arbitrary eigenvector associated with $\lambda_k^{(p)}$. Then $x$ induces at most $k+r-1$ strong and at most $k$ weak nodal domains.
\end{theorem}
The next lemma derives a general lower bound on the number of nodal domains of connected submodular hypergraphs. 
\begin{lemma}\label{atleasttwo}
Let $G$ be a connected submodular hypergraph. For $p>1$, any nonconstant eigenvector has at least two weak (strong) nodal domains. Hence, the eigenvectors associated with the second smallest eigenvalue $\lambda_2^{(p)}$ have exactly two weak (strong) nodal domains. For $p=1$, the eigenvectors associated with the second smallest eigenvalue $\lambda_2^{(1)}$ may have only one single weak (strong) nodal domain.
\end{lemma}
We define next the following three functions: 
\begin{align*}
\mu_p^+(x) &\triangleq \sum_{v\in V: x_v>0} \mu_v|x_v|^{p-1}, \\
\mu^0(x) &\triangleq \sum_{v\in V: x_v=0} \mu_v, \\
\mu_p^-(x) &\triangleq \sum_{v\in V: x_v<0} \mu_v |x_v|^{p-1}.
\end{align*}

The following lemma characterizes eigenvectors from another perspective that might be useful latter. 
\begin{lemma}\label{medianproperty}
Let $G$ be a connected submodular hypergraph. Then, for any nonconstant eigenvector $x$ of $\triangle_p$, one has $\mu_p^+(x) - \mu_p^-(x) = 0$ for $p>1$, and $|\mu_1^+(x) - \mu_1^-(x)| \leq \mu^0(x)$ for $p=1$. Consequently, $0 \in \arg\min_{c\in \mathbb{R}}\|x - c\mathbf{1}\|_{\ell_p,\mu}^p$ for any $p\geq 1$.
\end{lemma}
The nodal domain theorem characterizes the structure of the eigenvectors of the operator, and the number of nodal domains determines the approximation guarantees in Cheeger-type inequalities relating the spectra of graphs and hypergraphs and the Cheeger constant. These observations are rigorously formalized in the next section.

\subsection{Higher-Order Cheeger Inequalities}
In what follows, we analytically characterize the relationship between the Cheeger constants and the eigenvalues $\lambda_k^{(p)}$ of $\triangle_p$ for submodular hypergraphs. 
\begin{theorem} \label{cheeger}
Suppose that $p\geq 1$ and let $(\lambda_k^{(p)}, x_k)$ be the $k-$th eigenpair of the operator $\triangle_p$, with 
$m_k$ denoting the number of strong nodal domains of $x_k$. Then,
\begin{align*}
\left(\frac{1}{\tau}\right)^{p-1} \left(\frac{h_{m_k}}{p}\right)^p \leq \lambda_k^{(p)} \leq (\min\{{\zeta(E), k\}})^{p-1} \; h_k,
\end{align*}
where $\tau=\max_{v}\, d_v / \mu_v$.
For homogeneous hypergraphs, a tighter bound holds that reads as
\begin{align*}
\left(\frac{2}{\tau}\right)^{p-1} \left(\frac{h_{m_k}}{p}\right)^p \leq \lambda_k^{(p)} \leq 2^{p-1} \; h_k.
\end{align*}
\end{theorem}
It is straightforward to see that setting $p = 1$ produces the tightest bounds on the eigenvalues, while the case $p=2$ reduces to the classical Cheeger inequality. This motivates an in depth study of algorithms for evaluating the spectrum of $p=1,2$-Laplacians, described next.

\section{Spectral Clustering Algorithms for Submodular Hypergraphs} \label{sec:algs}

The Cheeger constant is frequently used as an objective function for (balanced) graph and hypergraph partitioning~\cite{zhou2007learning, buhler2009spectral,szlam2010total,hein2010inverse, hein2013total,li2017inhomogeneous}. Theorem~\ref{cheeger} implies that $\lambda_k^{(p)}$ is a good approximation for the $k$-way Cheeger constant of submodular graphs. Hence, to perform accurate hypergraph clustering, one has to be able to efficiently learn $\lambda_k^{(p)}$~\cite{ng2002spectral,von2007tutorial}. We outline next how to do so for $k=2$.

In Theorem~\ref{algoform}, we describe an objective function that allows us to characterize $\lambda_2^{(p)}$ in a computationally tractable manner; the choice of the objective function is related to the objective developed for graphs in~\cite{buhler2009spectral,szlam2010total}. Minimizing the proposed objective function produces a real-valued output vector $x\in \mathbb{R}^N$. Theorem~\ref{thresholding} describes how to round the vector $x$ and obtain a partition which provably upper bounds $c(S)$. Based on the theorems, we propose two algorithms for evaluating $\lambda_2^{(2)}$ and $\lambda_2^{(1)}$. Since $\lambda_2^{(1)}=h_2$, the corresponding partition corresponds to the tightest approximation of the $2$-way Cheeger constant. The eigenvalue $\lambda_2^{(2)}$ can be evaluated in polynomial time with provable performance guarantees. The problem of devising good approximations for values $\lambda_k^{(p)}$, $k \neq 2$, is still open. 

Let $Z_{p, \mu}(x,c) \triangleq \|x - c\mathbf{1}\|_{\ell_p,\mu}^p$ and $Z_{p, \mu}(x) \triangleq \min_{c\in \mathbb{R}} Z_{p, \mu}(x,c)$, and define 
\begin{align} \label{eq:rdef}
\mathcal{R}_p(x) \triangleq \frac{Q_p(x)}{Z_{p, \mu}(x)}.
\end{align}
\begin{theorem}\label{algoform}
For $p>1$, $\lambda_2^{(p)}=\inf_{x\in \mathbb{R}^N} \mathcal{R}_p(x)$. Moreover, $\lambda_2^{(1)} = \inf_{x\in \mathbb{R}^N} \mathcal{R}_1(x) = h_2$.
\end{theorem}
\begin{definition}
Given a nonconstant vector $x\in \mathbb{R}^N$, and a threshold $\theta$, set $\Theta(x, \theta) = \{v: x_{v} > \theta\}$. The optimal conductance obtained from thresholding vector $x$ equals 
\begin{align}
c(x) = \inf_{\theta \in [x_{\min}, x_{\max}) } \frac{\text{vol}(\partial \Theta(x, \theta))}{\min\{\text{vol}(\Theta(x, \theta)),\text{vol}(V/\Theta(x, \theta))\}}. \notag
\end{align} 
\end{definition}
\begin{theorem}\label{thresholding}
For any $x\in \mathbb{R}^N$ that satisfies $0 \in \arg\min_c Z_{p,\mu}(x,c)$, i.e., such that $Z_{p,\mu}(x,0) = Z_{p,\mu}(x)$, one has $c(x) \leq p\, \tau^{(p-1)/p} \, \mathcal{R}_p(x)^{1/p}$, where $\tau = \max_{v\in V} d_v/\mu_v$. 
\end{theorem}

In what follows, we present two algorithms. The first algorithm describes how to minimize $\mathcal{R}_2(x)$, and hence provides a polynomial-time solution for submodular hypergraph partitioning with provable approximation guarantees, given that the size of the largest hyperedge is a constant. The result is concluded in Theorem~\ref{SDPfinal}. The algorithm is based on an SDP, and may be computationally too intensive for practical applications involving large hypergrpahs of even moderately large hyperedges. The second algorithm is based on IPM~\cite{hein2010inverse} and aims to minimize $\mathcal{R}_1(x)$. Although this algorithm does not come with performance guarantees, it provably converges (see Theorem~\ref{convergence}) and has good heuristic performance. Moreover, the inner loop of the IPM involves solving a version of the proximal-type decomposable submodular minimization problem (see Theorem~\ref{diffprob}), which can be efficiently performed using a number of different algorithms~\cite{kolmogorov2012minimizing,jegelka2013reflection, nishihara2014convergence, ene2015random, li2018revisiting}. 

\subsection{An SDP Method for Minimizing $\mathcal{R}_2(x)$}

The $\mathcal{R}_2(x)$ minimization problem introduced in Equation~(\ref{eq:rdef}) may be rewritten as 
\begin{align} \label{minimizeF2}
\min_{x: Ux \perp \mathbf{1}} \frac{Q_2(x)}{\|x\|_{\ell_2, \mu}^2},
\end{align}
where we observe that $Q_2(x) = \sum_{e\in E} \vartheta_e f_e^2(x) =\sum_{e\in E}\vartheta_e\max_{y\in \mathcal{E}(\mathcal{B}_e)} \langle y, x\rangle^2$. This problem is, in turn, equivalent to the nonconvex optimization problem
\begin{align}
\min_{x\in \mathbb{R}^{N}} &\sum_{e} \vartheta_e \left(\max_{y\in \mathcal{E}(\mathcal{B}_e)} \langle y, x \rangle\right)^2 \label{minimizeF22} \\
\quad \text{s.t.}&\sum_{v\in V} \mu_v x_v^2 = 1, \; \sum_{v\in V} \mu_v x_v = 0.  \nonumber
\end{align}
Following an approach proposed for homogeneous hypergraphs~\cite{louis2015hypergraph}, one may try to solve an SDP relaxation of~\eqref{minimizeF22} instead. To describe the relaxation, let each vertex $v$ of the graph be associated with a vector $x'_v \in \mathbb{R}^n$, $n\geq \zeta(E)$. The assigned vectors are collected into a matrix of the form $X = (x'_1,..,x'_N)$. The SDP relaxation reads as
 \begin{align}
\min_{X\in \mathbb{R}^{n\times N},\;\eta \in\mathbb{R}^{|E|} } &\sum_{e} \vartheta_e \eta_e^2 \label{minimizeF23} \\
\text{s.t.}\quad  \|Xy\|_2^2 \leq \eta_e^2& \quad \forall y \in \mathcal{E}(\mathcal{B}_e), e\in E   \nonumber \\
  \sum_{v\in V} \mu_v \|x'_v\|_2^2 = 1 &, \sum_{v\in V} \mu_v x'_v = 0. \nonumber
\end{align}
Note that $\mathcal{E}(\mathcal{B}_e)$ is of size $O(|e|!)$, and the above problem can be solved efficiently if $\zeta(E)$ is small. 

Algorithm 1 lists the steps of an SDP-based algorithm for minimizing $\mathcal{R}_2(x)$, and it comes with approximation guarantees stated in Lemma~\ref{SDPapprox}. In contrast to homogeneous hypergraphs~\cite{louis2015hypergraph}, for which the approximation factor equals $O(\log \zeta(E))$, the guarantees for general submodular hypergraphs are $O(\zeta(E))$. This is due to the fact that the underlying base polytope $\mathcal{B}_e$ for a submodular function is significantly more complex than the corresponding polytope for the homogeneous case. We conjecture that this approximation guarantee is optimal for SDP methods.

 \begin{table}[htb]
\centering
\begin{tabular}{l}
\hline
\label{approximation}
\textbf{Algorithm 1: } \textbf{Minimization of $\mathcal{R}_2(x)$ using SDP} \\
\hline
\ \textbf{Input}: A submodular hypergraph $G=(V,E, \mathbf{w},\boldsymbol{\mu})$\\
\ 1: Solve the SDP \eqref{minimizeF23}. \\
\ 2: Generate a random Gaussian vector $g\sim N(0, I_n)$, \\
where $I_n$ denotes the identity matrix of order $n$. \\
\ 3: Output $x = X^T g$.\\
\hline
\end{tabular}
\end{table}

\begin{lemma}\label{SDPapprox}
Let $x$ be as in Algorithm 1, and let the optimal value of~\eqref{minimizeF23} be SDPopt. Then, with high probability, $\mathcal{R}_2(x) \leq O(\zeta(E)) \, \text{SDPopt} \, \leq O(\zeta(E)) \, \min  \mathcal{R}_2$.
\end{lemma}

This result immediately leads to the following theorem.
\begin{theorem}\label{SDPfinal}
Suppose that $x$ is the output of Algorithm 1. Then, $c(x) \leq  O(\sqrt{\zeta(E) \tau \, h_2})$ with high probability.
\end{theorem}
We describe next Algorithm 2 for optimizing $\mathcal{R}_1(x)$ which has guaranteed convergence properties.
\begin{table}[htb]
\centering
\begin{tabular}{l}
\hline
\label{IPM}
\textbf{Algorithm 2: } \textbf{IPM-based minimization of $\mathcal{R}_1(x)$} \\
\hline
\ \textbf{Input}: A submodular hypergraph $G=(V,E, \mathbf{w},\boldsymbol{\mu})$ \\
Find nonconstant $x^0\in \mathbb{R}^{N}$ s.t. $0 \in \arg\min_c \| x^0 - c\mathbf{1}\|_{\ell_1,\mu}$\\
\quad\quad initialize  $\hat{\lambda}^0 \leftarrow \mathcal{R}_1(x^{0})$, $k\leftarrow 0$ \\
\ 1: \textbf{Repeat}:\\
\ 2:  For $v\in V$, $g_v^{k} \leftarrow \left\{\begin{array}{cc} \text{sgn}(x_v^{k})\mu_v, &\text{if $x_v^{k}\neq 0$} \\ -\frac{\mu_1^+(x^{k}) -\mu_1^-(x^{k})}{\mu^0(x^{k})}\mu_v, &\text{if $x_v^{k}= 0$}  \end{array} \right. $ \\
\ 3:  $z^{k+1} \leftarrow \arg\min_{z:  \|z\|\leq 1} Q_1(z) - \hat{\lambda}^k \langle z, g^k\rangle$ \\
\ 4:   $c^{k+1} \leftarrow \arg\min_c \| z^{k+1} - c\mathbf{1}\|_{\ell_1,\mu}$ \\
\ 5:  $x^{k+1} \leftarrow  z^{k+1}  - c^{k+1}\mathbf{1}$ \\
\ 6:  $\hat{\lambda}^{k+1} \leftarrow \mathcal{R}_1(x^{k+1})$ \\
\ 7: {Until} $|\hat{\lambda}^{k+1} -\hat{\lambda}^{k}|/\hat{\lambda}^k <\epsilon$ \\
\ 8. \textbf{Output} $x^{k+1}$ \\
\hline
\end{tabular}
\end{table}
\begin{theorem}\label{convergence}
The sequence $\{x^{k}\}$ generated by Algorithm 2 satisfies $\mathcal{R}_1(x^{k+1}) \leq \mathcal{R}_1(x^{k})$. 
\end{theorem}
The computationally demanding part of Algorithm 2 is the optimization procedure in Step 3. The optimization problem is closely related to the problem of submodular function minimization (SFM) due to the defining properties of the Lov$\acute{\text{a}}$sz extension. Theorem~\ref{diffprob} describes different equivalent formulations of the optimization problem in Step 3. 

\begin{theorem}\label{diffprob}
If the norm of the vector $z$ in Step 3 is $\|z\|_{2}$, the underlying optimization problem is the dual of the following $\ell_2$ minimization problem
\begin{align} \label{DualDSFM}
\min_{y_e} \|\sum_{e\in E} y_e -\hat{\lambda}^k g^k\|_2^2, \quad \text{$y_e\in \vartheta_e \mathcal{B}_e$, $\,\forall \, e\in E$},
\end{align}
where the primal and dual variables are related according to $z = \frac{\hat{\lambda}^k g^k - \sum_{e\in E} y_e}{\|\hat{\lambda}^k g^k - \sum_{e\in E} y_e\|_2} $. 

If the norm of the vector $z$ in Step 3 is $\|z\|_{\infty}$, the underlying optimization problem is equivalent to the following SFM problem
\begin{align} \label{discreteDSFM}
\min_{S\subseteq V} \sum_{e} \vartheta_e w_e(S) - \hat{\lambda}^k  g^k(S),
\end{align}
where the the primal and dual variables are related according to $z_v = 1$ if $v\in S,$ and $z_v=-1$ if $v\notin S$.
\end{theorem}

\begin{figure*}[t]
\centering
\includegraphics[trim={0cm 0cm 0cm 0cm},clip,width=.24\textwidth]{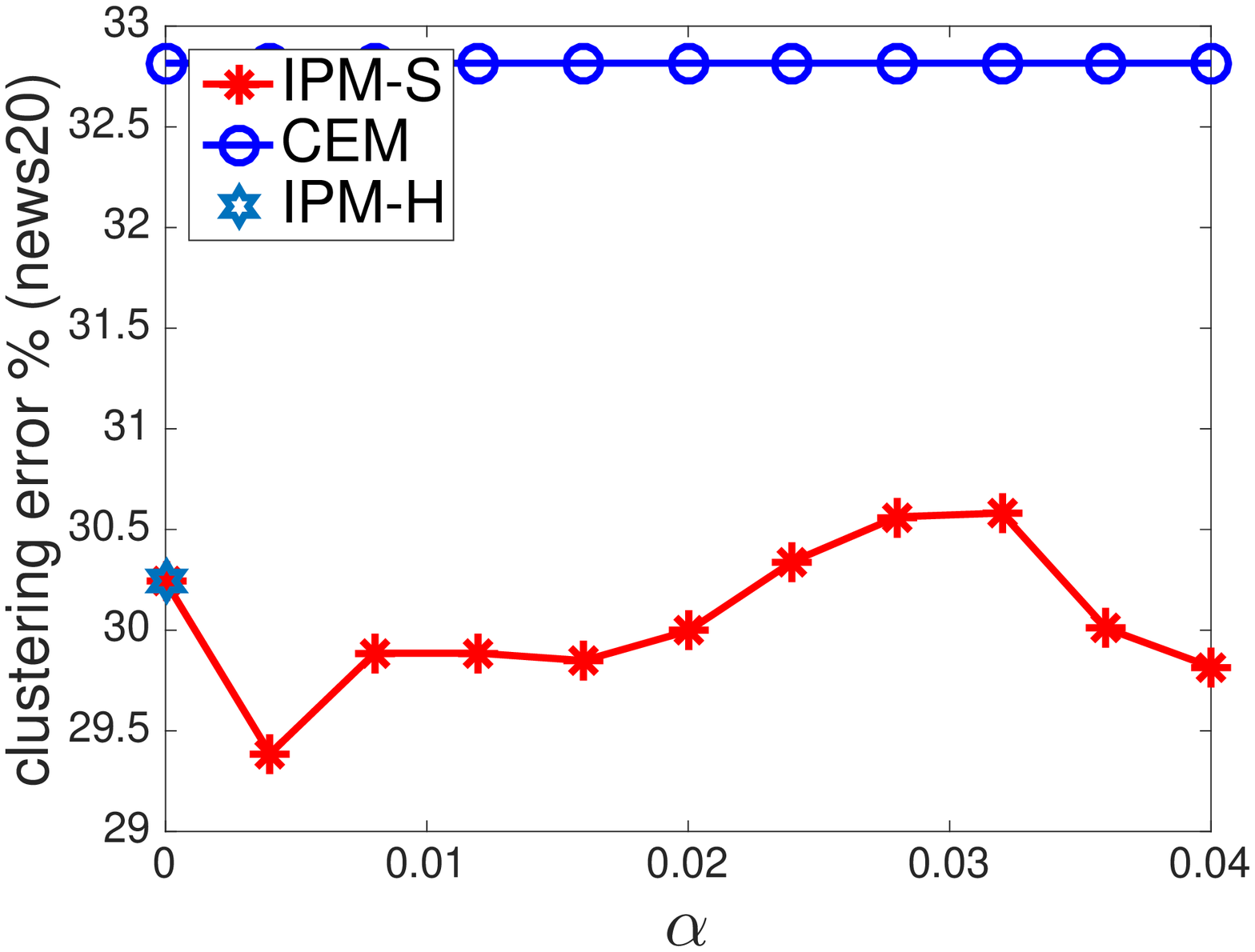}
\includegraphics[trim={0cm 0cm 0cm 0cm},clip, width=.24\textwidth]{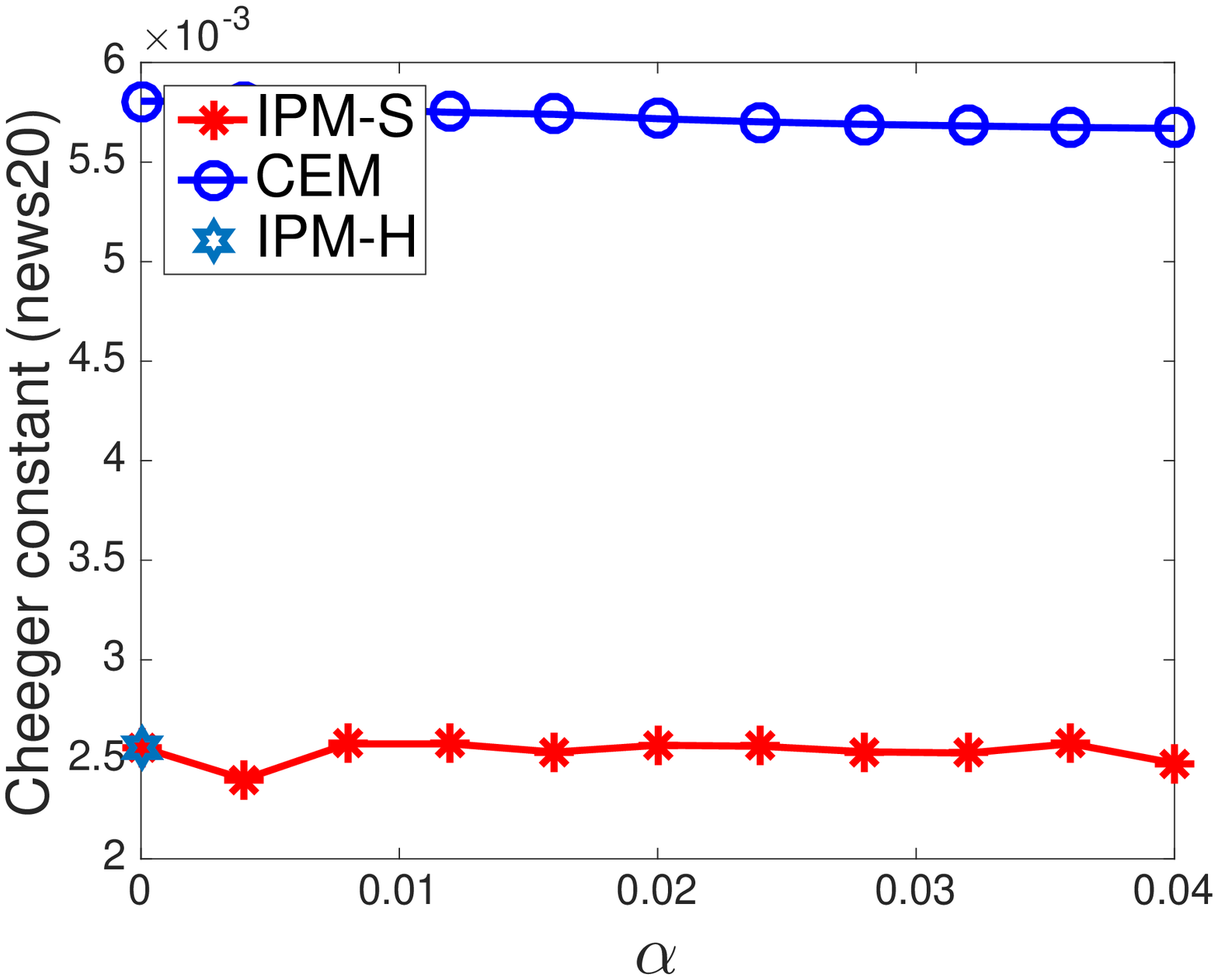}
\includegraphics[trim={0cm 0cm 0cm 0cm},clip,width=.24\textwidth]{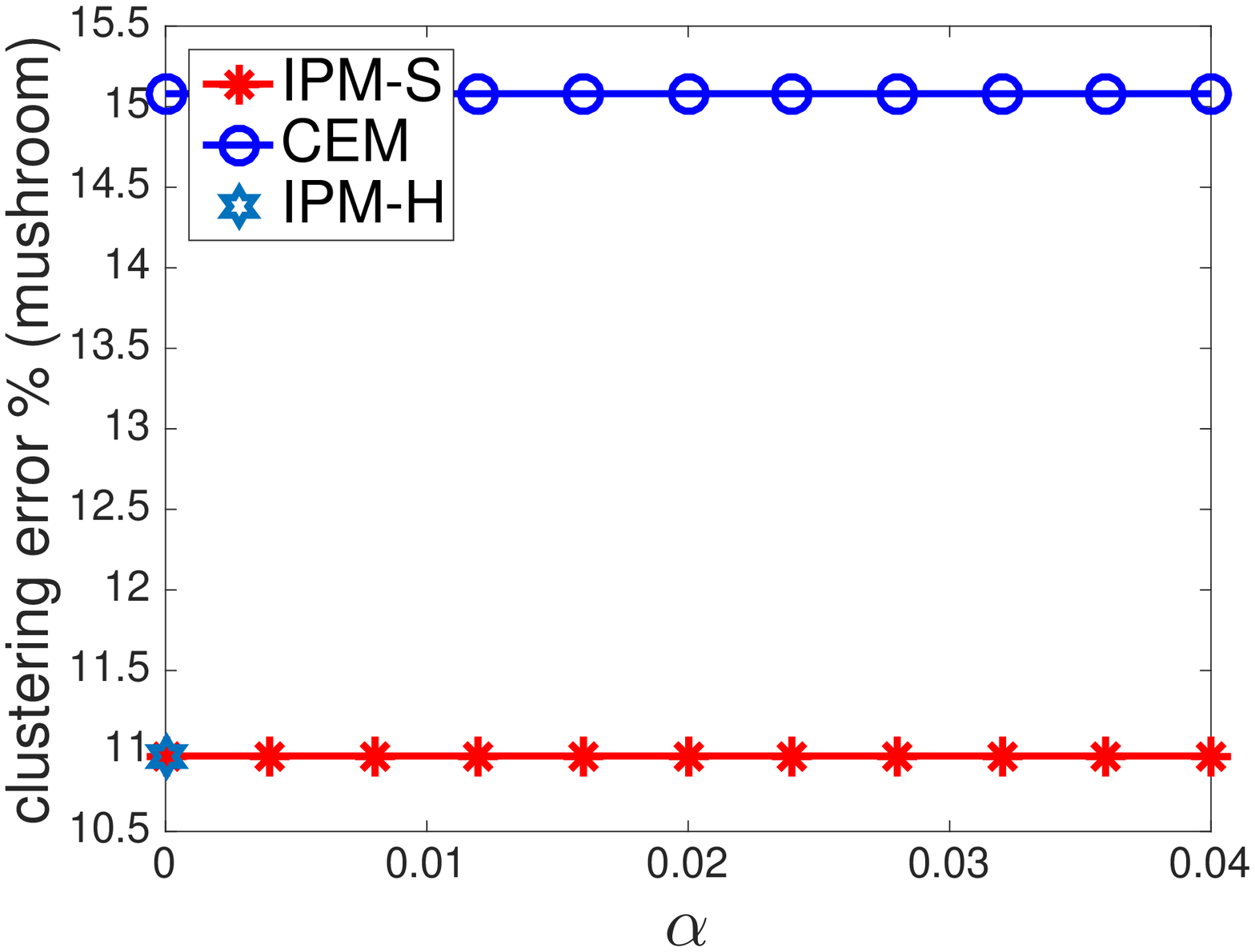}
\includegraphics[trim={0cm 0cm 0cm 0cm},clip, width=.24\textwidth]{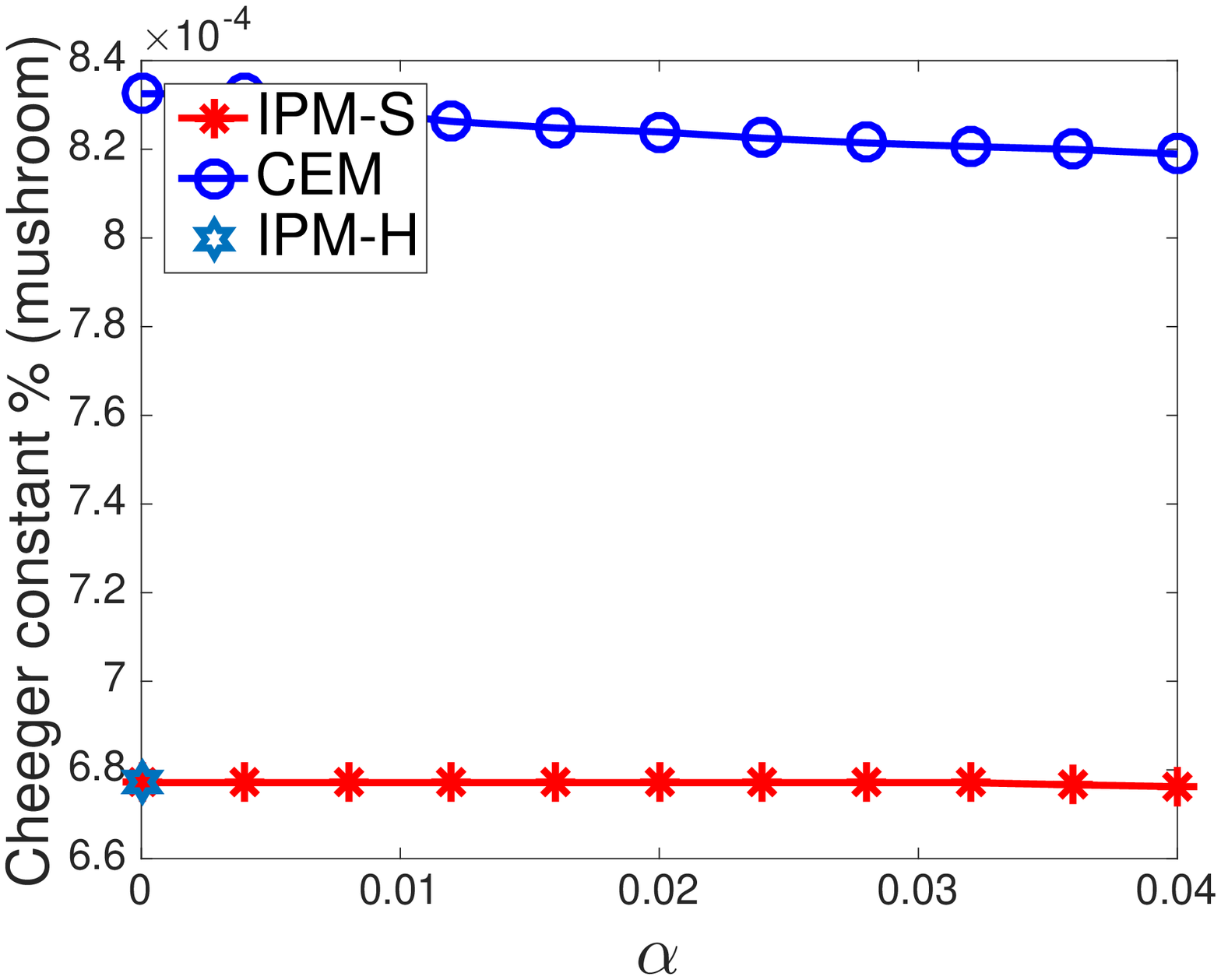}\\
\includegraphics[trim={0cm 0cm 0cm 0cm},clip,width=.24\textwidth]{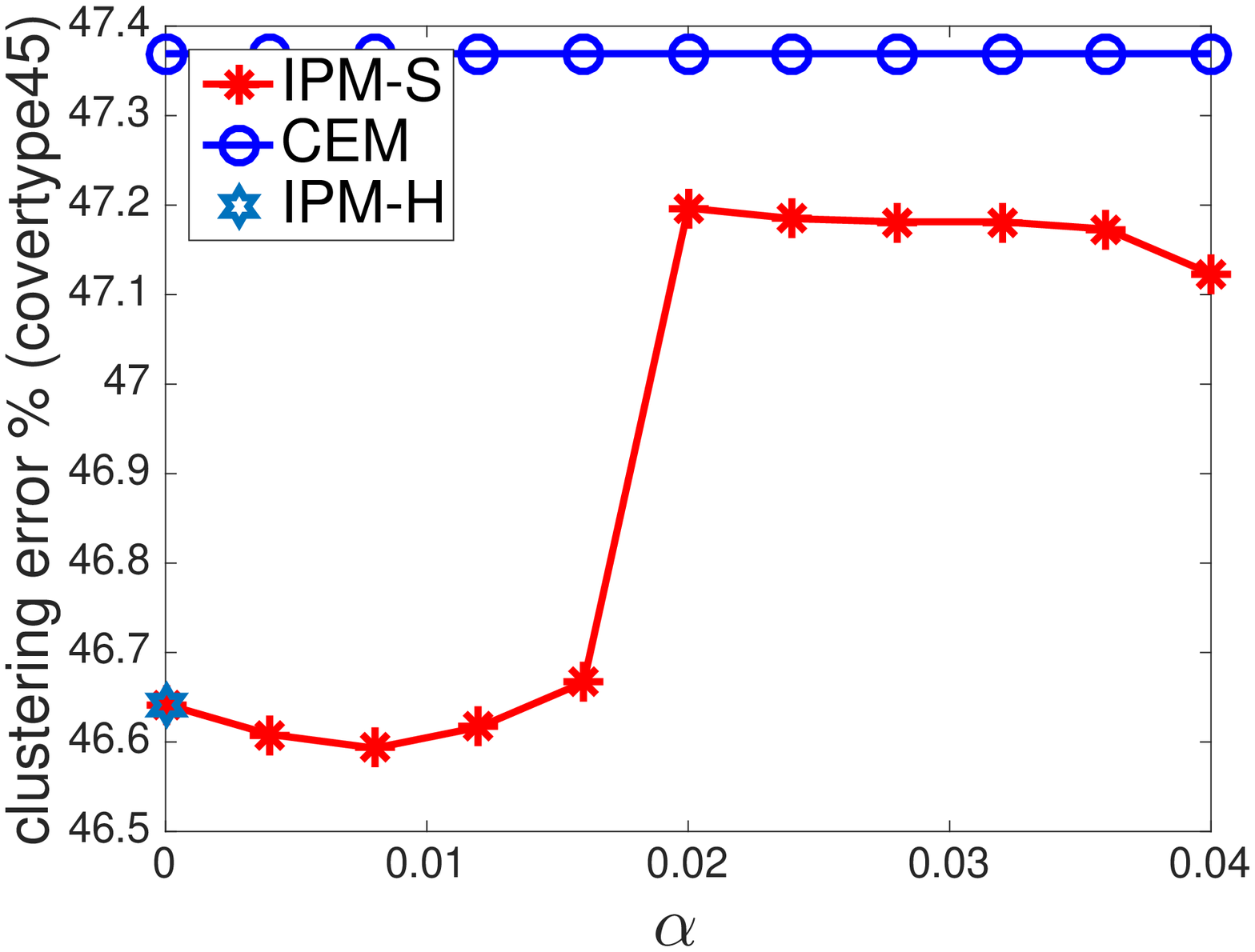}
\includegraphics[trim={0cm 0cm 0cm 0cm},clip, width=.24\textwidth]{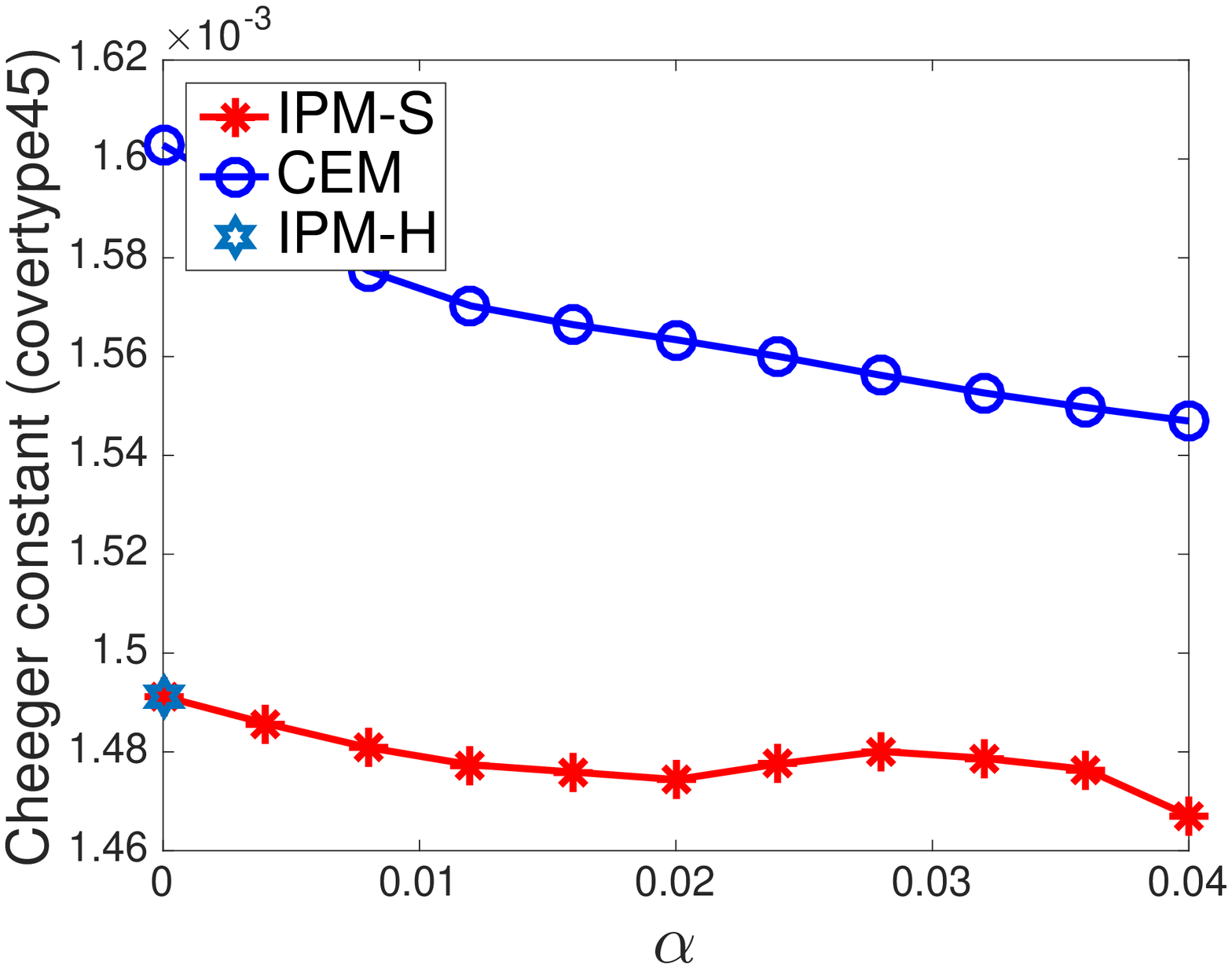}
\includegraphics[trim={0cm 0cm 0cm 0cm},clip,width=.24\textwidth]{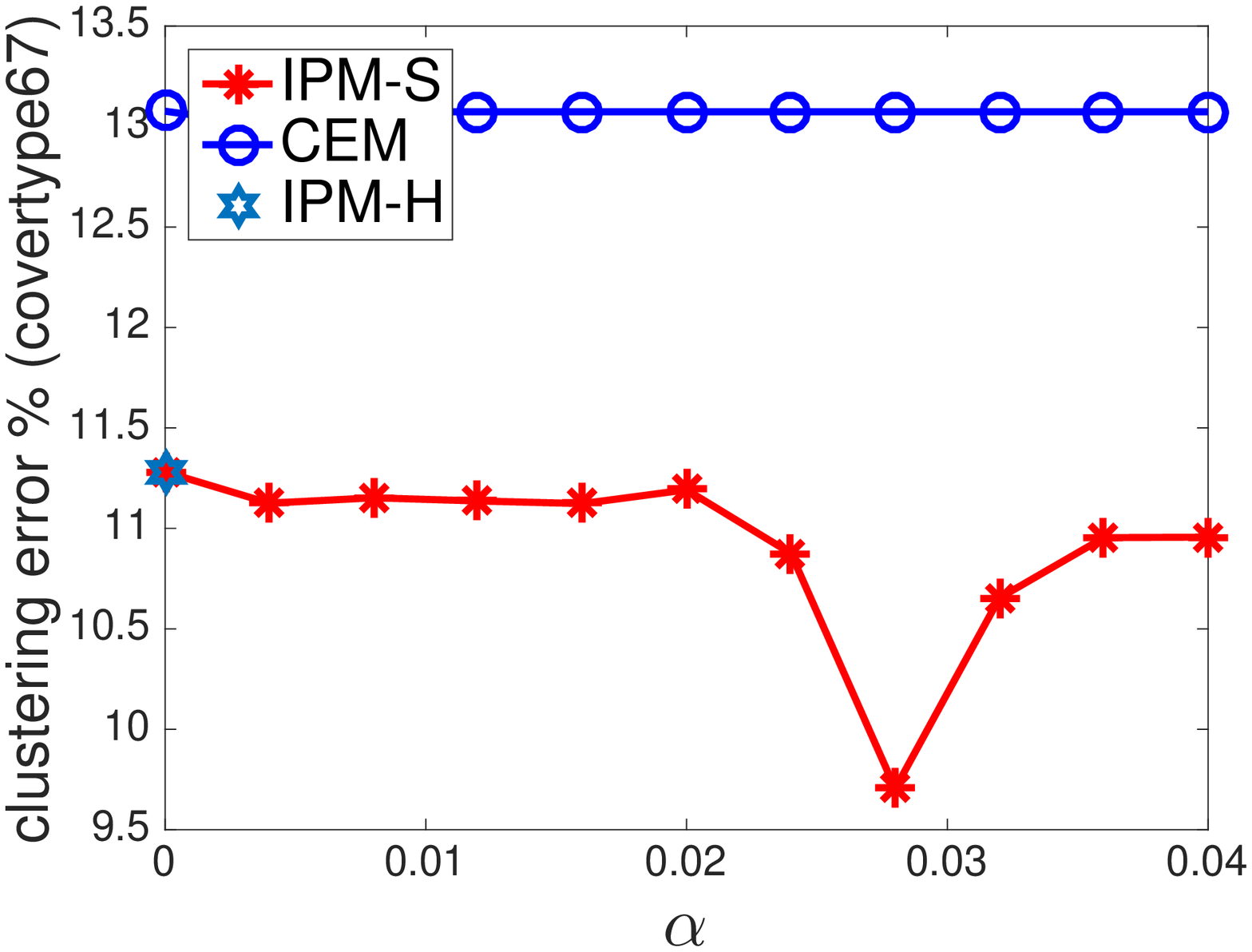}
\includegraphics[trim={0cm 0cm 0cm 0cm},clip, width=.24\textwidth]{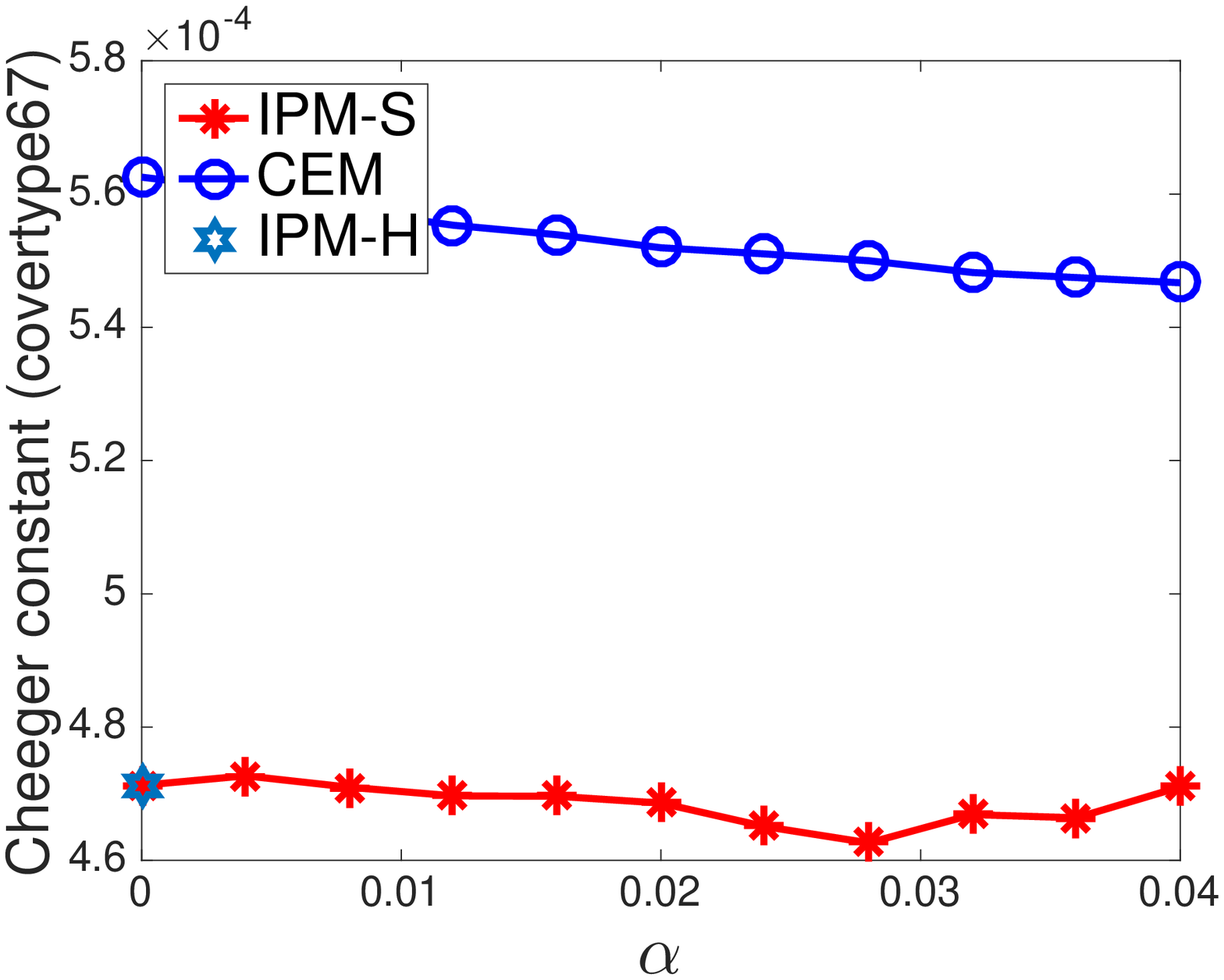}
\caption{Experimental clustering results for four UCI datasets, displayed in pairs of figures depicting the Clustering error and the Cheeger constant versus $\alpha$. Fine tuning the parameter $\alpha$ may produce significant performance improvements in several datasets - for example, on the Covertype67 dataset, choosing $\alpha=0.028$ results in visible drops of the clustering error and the Cheeger constant. \emph{Both the use of $1$-Laplacians and submodular weights may be credited for improving clustering performance.}}
\label{expresult}
\end{figure*}

For special forms of submodular weights, different algorithms for the optimization problems in Theorem~\ref{diffprob} may be used instead. For graphs and homogeneous hypergraphs with hyperedges of small size, the min-cut algorithms by Karger et al. and Chekuri et al.~\cite{karger1993global,chekuri2017computing} allow one to efficiently solve the discrete problem~\eqref{discreteDSFM}. 
Continuous optimization methods such as alternating projections (AP)~\cite{nishihara2014convergence} and coordinate descend methods (CDM)~\cite{ene2015random} can be used to solve~\eqref{DualDSFM} by ``tracking'' minimum norm points of base polytopes corresponding to individual hyperedges, where for general submodular weights, Wolfe's Algorithm~\cite{wolfe1976finding} can be used. When the submodular weights have some special properties, such as that they depend only on the cardinality of the input, there exist algorithms that operate efficiently even when $|e|$ is extremely large~\cite{jegelka2013reflection}. 


In our experimental evaluations, we use a random coordinate descent method (RCDM)~\cite{ene2015random}, which ensures an expected $(1+\epsilon)-$approximation by solving an expected number of $O(|V|^2|E|\log\frac{1}{\epsilon})$ min-norm-point problems. Note that when performing continuous optimization, one does not need to solve the inner-loop optimization problem exactly and is allowed to exit the loop as long as the objective function value decreases. Algorithm 3 lists the step of a RCDM algorithm in which one submodular hyperedge is sampled in one iteration, and the corresponding value of $y_e$ is updated (Clearly, multiple values of $y_e$ can be updated simultaneously if and only if the corresponding hyperedges do not intersect, and this parallelization step further improves the convergence rate of the method). 

\begin{table}[htb]
\centering
\begin{tabular}{l}
\hline
\label{RCDM}
\textbf{Algorithm 3: } \textbf{A RCDM for Solving the problem (8)} \\
\hline
\ \textbf{Input}: Submodular hypergraph $G=(V,E, \mathbf{w},\boldsymbol{\mu})$,  $\hat{\lambda}^k$, $g^k$.\\
\ 0: Initialize $y_e^{0}\in \vartheta_e\mathcal{B}_e$ for $e\in E$, $k\leftarrow 0$\\
\ 1: In iteration $k$:\\
\ 2: Sample one hyperedge $e\in E$ uniformly at random.\\
\ 3: $y_e^{k+1}\leftarrow \arg\min_{y_e\in \vartheta_e\mathcal{B}_e} \|y_e + \sum_{e'\in E/\{e\}} y_{e'} - \hat{\lambda}^k g^k\|_{2}^2$ \\
\ 4: Set $y_{e'}^{k+1}\leftarrow y_{e'}^{k}$ for $e'\neq e$.  \\
\ \textbf{Output}  $\frac{\hat{\lambda}^k g^k - \sum_{e\in E} y_e}{\|\hat{\lambda}^k g^k - \sum_{e\in E} y_e\|_2}$ \\
\hline
\end{tabular}
\end{table}

\section{Experiments} \label{sec:data}


In what follows, we compare the algorithms for submodular hypergraph clustering described in the previous section to two methods: The IPM for homogeneous hypergraph clustering~\cite{hein2013total} and the clique expansion method (CEM) for submodular hypergraph clustering~\cite{li2017inhomogeneous}. We focus on $2$-way graph partitioning problems related to the University of California Irvine (UCI) datasets selected for analysis in~\cite{hein2013total}, described in 
Table~\ref{tab:dataset}. The datasets include 20Newsgroups, Mushrooms, Covertype. In all datasets, $\zeta(E)$ was roughly $10^3$, and each of these datasets describes multiple clusters. Since we are interested in $2$-way partitioning, we focused on two pairs of clusters in Covertype, denoted by $(4,5)$ and $(6,7),$ and paired the four 20Newsgroups clusters, one of which includes \emph{Comp.} and \emph{Sci}, and another one which includes \emph{Rec.} and \emph{Talk}. The Mushrooms and 20Newsgroups datasets contain only categorical features, while Covertype also includes numerical features. We adopt the same approach as the one described in~\cite{hein2013total} to construct hyperedges: Each feature corresponds to one hyperedge; hence, each categorical feature is captured by one hyperedge, while numerical features are first quantized into $10$ bins of equal size, and then mapped to hyperedges. To describe the submodular weights, we fix $\vartheta_e = 1$ for all hyperedges and parametrize $w_e$ using a variable $\alpha\in (0, 0.5]$ 
\begin{align} 
w_e(S;\alpha) = \frac{1}{2}+ \frac{1}{2} \min\left\{1, \frac{|S|}{\lceil \alpha |e| \rceil }, \frac{|e/S|}{\lceil \alpha |e| \rceil }\right\}, \; \forall S\subseteq e. \notag
\end{align}
The intuitive explanation behind our choice of weights is that it allows one to accommodate categorization errors and outliers: In contrast to the homogeneous case in which any partition of a hyperedge has weight one, the chosen submodular weights allow a smaller weight to be used when the hyperedge is partitioned into small parts, i.e., when $\min \{|S|, |e/S|\}< \lceil \alpha |e| \rceil$. In practice, $\alpha$ is chosen to be relatively small -- in all experiments, we set $\alpha \leq 0.04$, with $\alpha$ close to zero producing homogeneous hyperedge weights. 

\begin{table*}[h] 
\begin{tabular}{|p{1.3cm}<{\centering}|p{2.2cm}<{\centering}|p{1.6cm}<{\centering}|p{2.6cm}<{\centering}|p{2.6cm}<{\centering}|}
\hline
Dataset & 20Newsgroups & Mushroom & Covertype$45$ & Covertype$67$ \\
\hline
$|V|$ & 16242 & 8124 & 12240 & 37877 \\
$|E|$ & 100 & 112 &  127 & 136 \\
$\sum_{e\in E} |e|$ & 65451 & 170604 & 145999 & 451529 \\
\hline
\end{tabular}
\centering 
\caption{The UCI datasets used for experimental testing.}
\label{tab:dataset}
\end{table*}

The results are shown in Figure~\ref{expresult}. As may be observed, both in terms of the Clustering error (i.e., the total number of erroneously classified vertices) and the values of the  Cheeger constant, IPM-based methods outperform CEM. This is due to the fact that for large hyperedge sizes, CEM incurs a high distortion when approximating the submodular weights ($O(\zeta(E))$~\cite{li2017inhomogeneous}). Moreover, as $w_e(S)$ depends merely on $|S|$, the submodular hypergraph CEM reduces to the homogeneous hypergraph CEM~\cite{zhou2007learning}, which is an issue that the IPM-based method does not face. Comparing the performance of IPM on submodular hypergraphs (IPM-S) with that on homogeneous hypergraphs (IPM-H), we see that IPM-S achieves better clustering performance on both 20Newsgroups and Covertypes, and offers the same performance as IPM-H on the Mushrooms dataset. This indicates that it is practically useful to use submodular hyperedge weights for clustering purposes. A somewhat unexpected finding is that for certain cases, one observes that when $\alpha$ increases (and thus, when $w_e$ decreases), the corresponding Cheeger constant increases. This may be caused by the fact that the IPM algorithm can get trapped in a local optima.

\bibliography{example_paper}
\bibliographystyle{IEEETran}

\newpage
\begin{appendix}

\begin{center}

{\Large \textbf{Appendix}}
\end{center}
\end{appendix}
\def\thesection{\Alph{section}}

\section{Preliminary Proofs}
We find the following properties of the Lov{\'a}sz extension of normalized symmetric submodular functions useful in the derivations to follow.
\begin{lemma}\label{basicprop}
Consider two vectors $x, x' \in \mathbb{R}^N$. If $F$ is a symmetric submodular function with $F([N]) = 0$, and $f(x)$ is the corresponding Lov{\'a}sz extension, then for any scalar $c \in \mathbb{R}$,
\begin{itemize}
\item[1)] $f(cx) = |c|f(x)$.
\item[2)] $ \nabla f(cx) = \text{sgn}(c)\nabla f(x)$, where $\text{sgn}$ denotes the sign function defined in the main text.
\item[3)] $\langle \nabla f(x), \mathbf{1}\rangle = 0$.
\end{itemize}
\end{lemma}
\begin{proof}
Given the definition of the Lov{\'a}sz extension and its subgradient, for any $c>0$ we have $f(cx) = cf(x)$ and $\nabla f(cx) = \nabla f(x)$. As $F$ is a symmetric submodular function, $f(x) = f(-x)$ is even, which establishes the first claim. Also, since $f(x)$ is even, $\nabla f(x)$ is odd, and thus, for some $c<0$, we have $\nabla f(cx) = \nabla f((-c)-x) = \nabla f(-x) = -\nabla f(x)$. For $c=0$,  $\nabla f(0) = \mathcal{B} = [-1,1] \mathcal{B}=\{{ab: a\in [-1,1], b\in \mathcal{B}\}},$ since $F$ is a symmetric submodular function. Hence, the second claim holds as well. The third claim follows from $\langle \nabla f(x), \mathbf{1}\rangle = F([N]) = 0$.
\end{proof}

\begin{definition}
Let $x, x'\in \mathbb{R}^N$. If $x_u > x_v \Rightarrow x'_u > x'_v$ for all $u, v \in [N]$, we write $x\rightharpoonup x'$.
\end{definition}
\begin{lemma}\label{dotproductpre}
Assume that $F$ is a submodular function defined on $[N]$ and that $f$ is its corresponding Lov{\'a}sz extension. If $x\rightharpoonup x'$, then $\nabla f(x') \subseteq \nabla f(x)$. Furthermore, $\langle \nabla f(x'), x\rangle = f(x)$.
\end{lemma}
\begin{proof}
Consider a point $y'\in \nabla f(x')$. According to Lemma 2.1, we know that $y'\in \arg\max_{y\in B} \langle y, x'\rangle$. Suppose that a nonincreasing order of components in $x'$ reads as $x_{i_1}' \geq x_{i_2}'\geq \cdots \geq x_{i_N}'$ . By the duality result of Proposition 3.2 in~\cite{bach2013learning}, it is known that $y'$ is an optimal solution to the above optimization problem if and only if $\sum_{j=1}^k y_{i_j}' = F(\{i_1,...,i_k\})$ whenever $x_{i_k}' > x_{i_{k+1}}'$ or $k= N$. As $x\rightharpoonup x'$, $\sum_{j=1}^k y_{i_j}' = F(\{i_1,...,i_k\})$ whenever $x_{i_k} > x_{i_{k+1}}$ or $k= N,$ and thus $y'$ is also an optimal solution for $\max_{y\in\mathcal{B}} \langle y, x\rangle$, i.e.,  $\nabla f(x') \subseteq \nabla f(x)$. Hence, $\langle \nabla f(x'), x\rangle\in \langle \nabla f(x), x\rangle = f(x),$ which concludes the proof.   
\end{proof}


\section{Proof for Equation~\eqref{subgradient}}
Suppose that $y'\in\arg\max_{y \in\mathcal{B}} \langle y, x \rangle$. Then, $f(x) = \langle y', x \rangle,$ and $f(x') \geq \langle y', x' \rangle$ for all $x'\in \mathbb{R}^N$. Therefore, $f(x')  - f(x)  \geq  \langle y', x'- x \rangle,$ and thus $y'$ is a subgradient of $f$ at $x$. 

Suppose next that $y'\in \nabla f(x)$, and let $S\subseteq [N]$. If $S = [N]$, we have $f(x\pm 1_{[N]}) \geq f(x) \pm \langle y', 1_{[N]} \rangle $. As $f(x\pm 1_{[N]}) = f(x)$, so $y'([N]) = 0$. When $S\neq [[N]]$, we have 
\begin{align*}
F(S) = & f(1_S)  = \max_{y\in \mathcal{B}} \langle y, 1_S \rangle =  \max_{y\in \mathcal{B}} \langle y, x+1_S - x\rangle \geq   \max_{y\in \mathcal{B}} \langle y, x+1_S \rangle -  \max_{y\in \mathcal{B}} \langle y, x \rangle \\
 = & f(x + 1_S) - f(x)  \stackrel{1)}{\geq}  \langle y',  x+1_S -x \rangle = y'(S),
\end{align*}
where $1)$ follows from the definition of the subgradient. Hence, $y'\in \mathcal{B}$. As $y'\in \nabla f(x)$, we have $f(\mathbf{0})-f(x) \geq \langle y', -x \rangle,$ which implies $\langle y', x \rangle \geq f(x)$. Hence, $y' \in \arg\max_{y \in \mathcal{B}} \langle y, x \rangle$.

\section{Proof for Theorem~\ref{eigenpairthm}}
We first prove Statement 1. Note that since
$$\nabla Q_p(x) = p \triangle_p(x),$$
$y\in \nabla Q_p(x)$ is equivalent to $y/p \in \triangle_p(x)$. 

When $p>1$, $\mathcal{S}_{p,\mu}$ is a differentiable and symmetric manifold. As $(\nabla \|x\|_{\ell_p, \mu}^p)_v = p\mu_v \phi_p(x_v)$, the tangent space of $\mathcal{S}_{p,\mu}$ at $x$ is a vector space that can be described as follows
\begin{align*}
T_x(\mathcal{S}_{p,\mu}) = \left\{\sum_{v\in [N]} c_v\chi_v,\;\text{where $\{c_v\}_{v\in [N]}$ satisfies}\; \sum_{v\in [N]} c_v \mu_v\phi_p(x_v) = 0\right \},
\end{align*}
where $\{\chi_v\}_{v\in [N]}$ is a canonical basis of $\mathbb{R}^N$. For a vector $y\in \nabla Q_p(x) $, its projection onto $T_x(\mathcal{S}_{p,\mu})$, i.e., $P_p(x)(y)$, vanishes if and only if $y\perp T_x(\mathcal{S}_{p,\mu})$. More precisely, 
 \begin{align*}
P_p(x)(y)=0 \Leftrightarrow \text{there exists some $c\in \mathbb{R}$ such that}\; y_v = c\mu_v\phi_p(x_v),  
\quad \text{for all $v\in [N]$,}
\end{align*}
which implies that $y\in \nabla Q_p(x)\cap  cU\phi_p(x)\neq \emptyset $. Therefore, $x$ is a critical point of $\tilde{Q}_p(x)$ if and only if $x$ is an eigenvector of $\triangle_p$. The corresponding eigenvalue is $\lambda = \frac{\langle x, \triangle_p x\rangle }{\langle x, U \phi_p(x) \rangle} = \frac{Q_p(x)}{\|x\|_{\ell_p,\mu}^p}= \tilde{Q}_p(x)$, i.e., the critical value of $\tilde{Q}_p$ at $x$.

When $p=1$, $\mathcal{S}_{p,\mu}$ is a piecewise linear manifold, whose tangent space at $x\in \mathcal{S}_{p,\mu}$ is a cone. According to Theorem 4.2 in~\cite{chang2016spectrum}, for some vector $y\in \nabla Q_p(x)$, its projection onto the tangent space at $x$, i.e., $P_p(x)(y)$, 
vanishes if and only if there exists some $c\in \mathbb{R}$ and $\{c_u\}$, where $|c_u|\leq 1$, such that
 \begin{align*}
y = c\left[\sum_{v: x_v\neq 0} \mu_v\text{sgn}(x_v)\chi_v + \sum_{u: x_u= 0} \mu_uc_u\chi_u\right], 
\end{align*}
which implies $y\in cU\phi_p(x) \cap \nabla Q_p(x) \neq \emptyset$. Therefore, $x$ is a critical point of $\tilde{Q}_p(x)$ if and only if $x$ is an eigenvector of $\triangle_p$. The corresponding eigenvalue is $\lambda = \frac{\langle x, \triangle_p x\rangle }{\langle x, U \phi_p(x) \rangle} = \frac{Q_p(x)}{\|x\|_{\ell_p,\mu}^p}= \tilde{Q}_p(x)$, i.e., the critical value of $\tilde{Q}_p$ at $x$.

Now we prove statements 2 and 3. For $p>1$, $\|x\|_{\ell_p, \mu}^p$ is differentiable, so
\begin{align}\label{derivativeofR}
\nabla R_p(x)  = \frac{ \|x\|_{\ell_p, \mu}^p \nabla Q_p(x)  - pQ_p(x) U \varphi_p(x) }{\|x\|_{\ell_p, \mu}^{2p}} = \frac{p}{\|x\|_{\ell_p, \mu}^{p}}\left(\triangle_p(x) - R_p(x)U \varphi_p(x)\right).
\end{align}
Hence, $0\in \nabla R_p(x)$ is equivalent to $0\in \triangle_p(x) \cap R_p(x)U \varphi_p(x)$, i.e., $(x, R_p(x))$ is an eigenpair. However, for $p=1$, we only have (See Proposition 2.3.14~\cite{clarke1990optimization})
\begin{align*}
\nabla R_p(x) \subseteq \frac{ \|x\|_{\ell_p, \mu}^p \nabla Q_p(x)  - pQ_p(x) U \varphi_p(x) }{\|x\|_{\ell_p, \mu}^{2p}}.
\end{align*}
Therefore, $0 \in$ the set on the right hand side does not necessarily imply that $0 \in \nabla R_p(x)$.

\section{Proof for Lemma \ref{reduce}}
The high level idea behind our proof is as follows: Given a hyperedge $e$, if for some nonempty $S \subset e$ we have $w_e(S) = 0$, then $e$ can be split into two hyperedges $e_1 = S$ and $e_2 = e\backslash S$ with two modified submodular weights associated with $e_1$ and $e_2$. As the size of $e$ is a constant, one can perform this procedure for all hyperedges $e$ until all nonempty subsets $S$ of $e$ satisfy $w_e(S) > 0$.
 
Consider a hyperedge $e$ with associated weight $w_e(S_1) = 0$ for some nonempty $S_1 \subset e$.  Then, for any $S\subseteq e$, it must hold that
\begin{align*}
2w_e(S) &\geq [w_e(S_1 \cup S) + w_e(S_1 \cap S)-w_e(S_1)] + [w_e(S_1\cup \bar{S})+ w_e(S_1\cap \bar{S})-w_e(S_1)] \\
              &=  [w_e(S_1 \cup S) + w_e(S_1 \cap S)] + [w_e(S\backslash S_1 )+ w_e(S_1\backslash S)] \\
              &= [w_e(S\cap S_1) + w_e(S \backslash S_1)] + [w_e(S\cup S_1) + w_e(S_1\backslash S)] \\
              &\geq 2w_e(S).
\end{align*}
Hence, all inequalities must be strict equalities so that
\begin{align*}
w_e(S) &= w_e(S_1 \cup S) + w_e(S_1 \cap S) = w_e(S_1\backslash S)+ w_e(S\backslash S_1) \\
            &= w_e(S\cap S_1) + w_e(S \backslash S_1) = w_e(S\cup S_1) + w_e(S_1\backslash S). 
\end{align*}
As a result, $w_e(S_1\backslash S) = w_e(S\cap S_1)$ and $w_e(S\backslash S_1) = w_e(S\cup S_1)$.  This implies that the hyperedge $e$ can be partitioned into two hyperedges, $e_1=S_1$ and $e_2=e\backslash S_1,$ with weights $(\vartheta_{e_i}, w_{e_i})_{i=1,2},$ such that 
\begin{align*}
\vartheta_{e_i} = \max_{S\subset e_i} w_e(S), \quad \vartheta_{e_i}w_{e_i}(S) = \vartheta_{e}w_e(S) \quad\text{for all $S\subset e_i$}.
\end{align*}
This partition ensures that $w_{e_i}$  is a normalized, symmetric submodular function and that for any $S\subseteq e$, $\vartheta_e w_e(S) = \vartheta_{e_1}w_{e_1}(S\cap e_1) + \vartheta_{e_2}w_{e_2}(S\cap e_2)$. Therefore, for any subset $S$ of $[N]$, the volume $\text{vol}(\partial S)$ remains unchanged. 


\section{Proof for Lemma~\ref{nontrivial}}
Let $x$ be an eigenvector associated with the eigenvalue $0$. Then, $Q_p(x) = \langle x, \triangle_p x \rangle=0$. Therefore, for each hyperedge $e$, we have $f_e(x)=0$. Based on Lemma~\eqref{reduce} of the main text, we may assume that the weights of $G$ have been transformed so that for any $e\in E$ and any set $S\cap e \neq \{\emptyset, e\}$, one has $w_e(S)>0$. Therefore, for any $v\in e$, $x_v$ is a constant vector. As in the transformed $G$, for each pair of vertices $v, u \in [N]$, one can find a hyperedge path from $v$ to $u$, so for all $v\in [N]$, $x_v$ is a constant vector 

\section{Proof for Discrete Nodal Domain Theorems}

The outline of the proof is similar to  the one given by Tudisco and Hein~\cite{tudisco2016nodal} for graph $p-$Laplacians, with one significant change that involves careful handling of submodular hyperedges. 

We start by introducing some useful notation. For a vector $x\in \mathbb{R}^N$ and a set $A\subset [N]$, define a vector $x|_A$ as
\begin{equation*}
(x|_A)_v=\left\{ \begin{array}{lc}
x_v & v\in A  \\
0 & v\not\in A   \end{array}\right.
\end{equation*}
We also define the strong (weak) nodal space $\Xi(x)$ (respectively, $\xi(x)$) induced by $x$ as the linear span of $x|_{A_1}, x|_{A_2}, \cdots, x|_{A_m}$, where $A_i, i=1,\ldots,m$ are the strong (weak) nodal domains of $x$.  
\begin{lemma}\label{weakstrongnodal}
A weak nodal space is a subspace of a strong nodal space. Hence, the number of weak nodal domains is upper bounded by the number of strong nodal domains for both $p=1$ and $p>1$ cases.
\end{lemma}
\begin{proof}
Suppose that the weak nodal domains of a vector $x$ equal $A_1, A_2, ..., A_m$. Hence, its weak nodal space equals to $\xi(x)=\{y| y = \sum_{i\in[m]}\alpha_i x|_{A_i}, \alpha_i\in \mathbb{R}\}$. Let $Z=\{v\in [N]: x_v=0 \}$ and set $C_i= A_i\backslash Z$ for $i\in[m]$. The subgraph in $G$ induced by the vertex set $C_i$ may contain several connected components, in which case one may further partition $C_i$ into disjoint sets $C_{i,1}, C_{i,2}, ..., C_{i,i_k}$, each of which corresponds to a connected component. It is easy to check that the strong nodal domains of $x$ exactly consist of $\{C_{i,j}\}_{1\leq i\leq m,  1\leq j\leq i_k}$. Therefore, the strong nodal space equals $\Xi(x)=\{y| y= \sum_{i,j}\alpha_{ij} x|_{C_{ij}}, \alpha_{ij}\in \mathbb{R}\}$ and contains $\xi(x)$.
\end{proof}
Our subsequent analysis of nodal domains is primarily based on the following three lemmas.

Based on the deformation theorems for locally Lipschitzian even functions on $\mathcal{S}_{\mu,1}$~(Theorem 4.8~\cite{chang2016spectrum}) and $\mathcal{S}_{\mu,p}$~($p>1$, Theorem 3.1~\cite{chang1981variational}), one can guarantee that each critical value corresponds to at least one critical point, which is described in the first lemma.
\begin{lemma}[Lemma 2.2~\cite{tudisco2016nodal}]\label{minimizingset}
For $k\geq 1$ and $p\geq 1$, let $A^*\in \mathcal{F}_k(\mathcal{S}_{p,\mu})$ be a minimizing set, i.e., a set such that
\begin{align*}
\lambda_k^{(p)} = \min_{A: \mathcal{F}_k(\mathcal{S}_{p,\mu})}\max_{x\in A} R_p(x)  = \max_{x\in A^*} R_p(x). 
\end{align*}
Then $A^*$ contains at least one critical point of $R_p(x)$ with respective to the critical value $\lambda_k^{(p)}$. 
\end{lemma}

\begin{lemma}[Lemma 3.7~\cite{tudisco2016nodal}]\label{ineqfornodal}
Let $p> 1,\,a,\,b,\,x,\,y\in \mathbb{R},$ so that $x, y\geq 0$. Then 
\begin{align*}
|ax+by|^p \leq (|a|^px+|b|^py)(x+y)^{p-1},
\end{align*}
where the equality if and only if $xy=0$ or $a=b$.
\end{lemma}
 \begin{lemma}\label{nodalanalysis}
 Let $p\geq 1$ and let $(x,\lambda)$ be an eigenpair of $\triangle_p$. Let $\Xi(x)$ ($\xi(x)$) be the strong (weak) nodal space induced by $x$. Then, for any vector $x'\in\Xi(x)$  ($\xi(x)$), it holds that $Q_p(x')\leq \lambda\|x'\|_{\ell_p,\mu}^p$, and the inequality is tight for $p=1$. 
 \end{lemma}
\begin{proof}
Due to Lemma~\ref{weakstrongnodal}, we only need to prove the claimed result for the strong nodal space. Suppose $A_1, A_2, ..., A_m$ are the strong nodal domains of $x$. Consider a vector in the strong nodal space of $x$, say $y=\sum_{i} \alpha_i x|_{A_i}$, where $\alpha_i\in\mathbb{R}$. The following observation is important when generalizing result pertaining to graphs to the case of submodular hypergraphs. As we assume that the submodular hypergraph $G$ is connected, we may without loss of generality assume that $G$ is a hypergraph obtained from the transform described in Lemma~\ref{reduce}. Then, based on the definition of nodal domains, each hyperedge $e$ intersects at most two strong nodal domains with different signs. Hence, $x|_{A_i\cap e} \rightharpoonup x|_e$ for any $i\in[m], e\in E$ and $x|_{A_i\cap e} \rightharpoonup \text{sgn}(\alpha_i) y|_{e}$ for any $i\in[m], \alpha_i\neq 0, e\in E$. From Lemma~\ref{dotproductpre}, and for any $c\in\mathbb{R}, i\in [m]$, one has
\begin{align}
\langle \nabla f_e(x), cx|_{A_i} \rangle  = c\langle \nabla f_e(x|_e), x|_{A_i\cap e}  \rangle = cf_e(x|_{A_i\cap e}) = cf_e(x|_{A_i}),  \label{expansion1}
\end{align}
and 
\begin{align}
f_e(y) = &\langle f_e(y), y \rangle =\left\langle \nabla f_e(y),  \sum_{i}\alpha_i x|_{A_i} \right\rangle= \sum_{i}\alpha_i \left\langle \nabla f_e(y),  x|_{A_i\cap e} \right\rangle \nonumber \\
 = &\sum_{i}\alpha_i\left\langle  \text{sgn}(\alpha_i) \nabla f_e(\text{sgn}(\alpha_i) y|_e),  x|_{A_i\cap e} \right\rangle = \sum_{i}|\alpha_i|  f_e(x|_{A_i\cap e})\nonumber \\
  = &  \sum_{i}|\alpha_i|  f_e(x|_{A_i}). \label{expansion}
\end{align}
We partition the hyperedges into two sets according to how many nodal domains they intersect, 
\begin{align*}
\mathcal{I}_1 &= \{e:| \{i | e \cap A_i \neq \emptyset\} |\leq 1 \}, \\
\mathcal{I}_2 &= \{e: |\{i | e \cap A_i \neq \emptyset\}| = 2 \}.
\end{align*}
Then, we have
\begin{align*}
Q_p(y) &= \sum_{e} \vartheta_e(f_e(y))^p \stackrel{1)}{=} \sum_{e}  \vartheta_e\left(\sum_{i}|\alpha_i|f_e(x|_{A_i})\right)^p   \\
&= \sum_{e\in \mathcal{I}_1}  \vartheta_e\sum_i |\alpha_i|^p (f_e(x|_{A_i}))^p +  \sum_{e\in \mathcal{I}_2}  \vartheta_e\left(\sum_{i}|\alpha_i|f_e(x|_{A_i})\right)^p\\
& \stackrel{2)}{=} \sum_{e\in \mathcal{I}_1}  \vartheta_e\sum_i |\alpha_i|^p f_e(x|_{A_i}) (f_e(x))^{p-1} +  \sum_{e\in \mathcal{I}_2}  \vartheta_e\left(\sum_{i}|\alpha_i|f_e(x|_{A_i})\right)^p,
\end{align*}
where $1)$ follows from \eqref{expansion} and $2)$ is due to the fact that $f_e(x) = f_e(x|_{A_i})$ for those $i$ such that $A_i\cap e\neq \emptyset$, and $f_e(x)=0$ for those $i$ such that $A_i\cap e = \emptyset$.
Moreover, we have 
\begin{align*}
\lambda\|y\|_{\ell_p,\mu}^p  &= \sum_{i} |\alpha_i|^p \lambda \|x|_{A_i}\|_{\ell_p,\mu}^p \stackrel{1)}{=}  \sum_{i} |\alpha_i|^p \langle x|_{A_i}, \triangle_p x \rangle \\
&= \sum_{i} |\alpha_i|^p \sum_{e}\vartheta_e \langle \nabla f_e(x), x|_{A_i}\rangle  (f_e(x))^{p-1} \\
& \stackrel{2)}{=}  \sum_{i} |\alpha_i|^p \sum_{e}\vartheta_e f_e(x|_{A_i}) (f_e(x))^{p-1} 
\end{align*}
where $1)$ is due to 
\begin{align*}
\lambda \|x|_{A_i}\|_{\ell_p,\mu}^p= \langle x|_{A_i},\lambda\varphi_p(x|_{A_i}) \rangle =\langle x|_{A_i},\lambda\varphi_p(x) \rangle=\langle x|_{A_i}, \triangle_p x \rangle, 
\end{align*}
and $2)$ follows from \eqref{expansion1}.
Therefore, 
\begin{align*}
Q_p(y) - \lambda\|y\|_{\ell_p,\mu}^p = \sum_{e\in \mathcal{I}_2} \vartheta_e\left[\left(\sum_{i}|\alpha_i|f_e(x|_{A_i})\right)^p -  \sum_{i} |\alpha_i|^p  f_e(x|_{A_i}) (f_e(x))^{p-1} \right]  =\sum_{e\in \mathcal{I}_2} \vartheta_e \tilde{f}_e(y),
\end{align*}
where 
 \begin{align*}
 \tilde{f}_e(y) = \left\{\left[|\alpha_{i_1}|f_e(x|_{A_{i_1}}) \right.\right.& \left.+ |\alpha_{i_2}|f_e(x|_{A_{i_2}})\right]^p - \\
& \left.\left[|\alpha_{i_1}|^p  f_e(x|_{A_{i_1}})+ |\alpha_{i_2}|^p  f_e(x|_{A_{i_2}})\right]\left[f_e(x|_{A_{i_1}}) + f_e(x|_{A_{i_2}})\right]^{p-1}\right\}
 \end{align*}
and $A_{i_1}$ and $A_{i_2}$ are the two nodal domains intersecting $e$. Invoking Lemma~\ref{ineqfornodal} proves the claimed result. 
\end{proof}

Now, we are ready to prove Theorem~\ref{nodal}. The proof of the strong nodal domain result for the graph $p-$Laplacian in~\cite{tudisco2016nodal} can be easily extended to our case via Lemma~\ref{nodalanalysis}, while the proof of the weak nodal domain result requires significant modifications. 

\textbf{Case 1: Strong nodal domains.} Suppose that $\lambda_k^{(p)}$ has multiplicity $r$ and associated eigenvector $x$. Let $\Xi(x)$ be the strong nodal space induced by $x$. If $x$ supports $m$ strong nodal domains,  then $\gamma(\Xi(x)\cap \mathcal{S}_{p,\mu})\leq m$. For any $x'\in \Xi(x)\cap \mathcal{S}_{p,\mu}$, we have $R_p(x') \leq R_p(x)=\lambda_k^{(p)}$ due to Lemma~\ref{nodalanalysis}. Therefore, 
\begin{align*}
\lambda_m^{(p)}= \min_{A \in \mathcal{F}_m(\mathcal{S}_{p,\mu})}\max_{x'\in A} R_p(x') \leq \max_{x'\in  \Xi(x)\cap \mathcal{S}_{p,\mu}} R_p(x') \leq \lambda_k^{(p)},
\end{align*}
which implies $m\leq k+r-1$, where $r$ is the mulplicity of $\lambda_k$. Given this upper bound of the number of strong nodal domains and Lemma~\ref{weakstrongnodal}, one may natural bound for the number of weak nodal domains. However, for $p>1$ case, one may derive a tighter bound.

\textbf{Case 2: Weak nodal domains (for $p>1$).} Suppose that $\lambda_k^{(p)}$ has multiplicity $r$ and associated eigenvector $x$. Suppose that $A_1, A_2, ..., A_m$ are the weak nodal domains of $x$.  According to Lemma~\ref{weakstrongnodal}, we know that $m$ is upper bounded by the number of strong nodal domains which we know from Case 1 to be upper bounded by $k+r-1$. 

Let $\xi(x)$ be the weak nodal space  induced by $x$. We use proof by contradiction and assume that $\text{dim($\xi(x)$)}> k$. Consider $\xi(x)'$ satisfying $\xi(x)=\text{Span}\{x\}\oplus \xi(x)'$.Then, we have $\gamma(\xi(x)'\cap \text{Span}\{x\})\geq k$. Again, from Lemma~\ref{nodalanalysis}, it holds
\begin{align*}
\lambda_k^{(p)}\leq \min_{A \in \mathcal{F}_k( \mathcal{S}_{p,\mu})}\max_{x'\in A} \tilde{Q}_p(x') \leq \max_{x'\in \xi(x)'\cap \mathcal{S}_{p,\mu}} \tilde{Q}_p(x') \leq \tilde{Q}_p(x)=\lambda_k^{(p)},
\end{align*}
which implies that $ \xi(x)'\cap \mathcal{S}_{p,\mu}$ is a minimizing set in $\mathcal{F}_k( \mathcal{S}_{p,\mu})$. From Lemma~\ref{minimizingset}, it follows that there exists a critical point $y\in  \xi(x)'\cap \mathcal{S}_{p,\mu}$ such that $\tilde{Q}_p(y)=\lambda_k^{(p)}$. Therefore, $y$ is also an eigenvector of $\triangle_p$ with respect to the eigenvalue $\lambda_k^{(k)}$. Suppose that $y=\sum_{i}\alpha_i x|_{A_i}$. Later, we will show the contradiction by proving that $y \in \text{Span}\{x\}$, i.e., $\alpha_i =\alpha_j$ for all $i,j\in[m]$. For any two overlapping weak nodal domains, say $A_1$ and $A_2$ with $A_1\cap A_2\neq \emptyset$, consider the set of hyperedges that lie in $A_1\cup A_2$, and denote this set by $E^*$. Without loss of generality, assume that $A_1$ is positive while $A_2$ is negative, as no hyperedge can intersect two weak nodal domains with the same sign. Suppose that there exists a hyperedge $e\in E^*$ such that $e \cap (A_1 \backslash  A_2)$ and $e \cap (A_2 \backslash  A_1)$ are both nonempty. Then, both $f_e(x|_{A_{1}})$ and $f_e(x|_{A_{2}})$ are positive. According to the proof of Lemma~\ref{nodalanalysis}, as $e$ intersects two strong nodal domains $A_1 \backslash  A_2$ and $A_2\backslash A_1$, in order to have $R_p(y) =\lambda_k^{(p)}$ one must also have $\tilde{f}_e(y)=0,$ which further implies $\alpha_1=\alpha_2$.
If there is no such hyperedge, then all hyperedges in $E^*$ lie either in $A_1$ or $A_2$. Note that for all $u\in A_1\cap A_2,\;\alpha_1x_u =0$, so that for $p>1$, we have 
\begin{align*}
0&=\lambda_k^{(p)}\mu_u\langle \mathbf{1}_u, \varphi_p({\alpha_1 x})\rangle=\langle \mathbf{1}_u, \triangle_p(\alpha_1 x)\rangle \\
&= \sum_{e: e\in E^*}\vartheta_e \langle \nabla f_e(\alpha_1 x), \mathbf{1}_u\rangle  (f_e(\alpha_1 x))^{p-1} \\
&\stackrel{1)}{=} \sum_{e: e\in A_1}\vartheta_e \langle \nabla f_e(\alpha_1 x|_{A_1}), \mathbf{1}_u\rangle  (f_e(\alpha_1 x|_{A_1}))^{p-1} +  \sum_{e: e\in A_2}\vartheta_e \langle \nabla f_e(\alpha_1 x|_{A_2}), \mathbf{1}_u\rangle  (f_e(\alpha_1 x|_{A_2}))^{p-1}, 
\end{align*}
where $1)$ is due to the fact that for all $e\subseteq A_1\cap A_2,$ one has $f_e(\alpha_1x)=0$. Similarly, as $y_u=0$ and $y$ is an eigenvector of $\triangle_p$, for $p>1$, we have 
\begin{align*}
0&=\lambda_k^{(p)}\mu_u\langle \mathbf{1}_u, \varphi_p({y})\rangle=\langle \mathbf{1}_u, \triangle_p(y)\rangle \\
&= \sum_{e: e\in E^*}\vartheta_e \langle \nabla f_e(y), \mathbf{1}_u\rangle  (f_e(y))^{p-1} \\
&\stackrel{1)}{=}  \sum_{e: e\subseteq A_1} \vartheta_e \langle \nabla f_e(\alpha_1 x|_{A_1}), \mathbf{1}_u\rangle (f_e(\alpha_1 x|_{A_1}))^{p-1} +  \sum_{e: e\subseteq A_2} \vartheta_e \langle \nabla f_e(\alpha_2 x|_{A_2}), \mathbf{1}_u\rangle (f_e(\alpha_2 x|_{A_2}))^{p-1},
\end{align*}
where $1)$ once again is due to for all $e\subseteq A_1\cap A_2,$ one has $f_e(\alpha_1x)=0$.
Subtracting the above two equations leads to
\begin{align*}
0&=\sum_{e: e\subseteq A_2} \vartheta_e\left[ \langle \nabla f_e(\alpha_1 x|_{A_2}), \mathbf{1}_u\rangle (f_e(\alpha_1 x|_{A_2}))^{p-1}- \langle \nabla f_e(\alpha_2 x|_{A_2}), \mathbf{1}_u\rangle (f_e(\alpha_2 x|_{A_2}))^{p-1}\right] \\
&=(\varphi_p(\alpha_1)- \varphi_p(\alpha_2)) \sum_{e: e\subseteq A_2} \vartheta_e \langle f_e(x|_{A_2\cap e}), \mathbf{1}_u\rangle (f_e(x|_{A_2}))^{p-1} \\
&\stackrel{1)}{=}(\varphi_p(\alpha_1)- \varphi_p(\alpha_2)) \sum_{e: e\subseteq A_2} \vartheta_e  f_e(\mathbf{1}_u) (f_e(x|_{A_2}))^{p-1},
\end{align*}
where $1)$ is due to $\mathbf{1}_u\rightharpoonup x|_{A_2\cap e}$. Based of the definition of a weak nodal domain, there exists at least one hyperedge $e$ intersecting both $A_1\cap A_2$ and $A_2\backslash A_1$. Therefore, for any $u\in A_1\cap A_2$ such that $\sum_{e: e\subseteq A_2} \vartheta_e  f_e(\mathbf{1}_u) f_e(x|_{A_2})^{p-1}>0$, one has $\varphi_p(\alpha_1)= \varphi_p(\alpha_2)$ and consequently $\alpha_1=\alpha_2$. Since the hypergraph is connected, it follows that $\alpha_1 = \alpha_2 =...=\alpha_n=\alpha$, which implies that $y=\alpha x$. This is a contradiction and hence when $p>1$, the number of weak nodal domains is $\leq k$. 

Note that example 10~\cite{chang2017nodal} shows that even for graphs, the number of weak nodal domains of an eigenvector of $\lambda_k^{(1)}$ can be greater than $k$. 

\subsection{Proof for Lemma~\ref{atleasttwo}}
Consider a nonconstant eigenvector $x$ and its corresponding eigenvalue $\lambda$. According to Lemma~\ref{nontrivial}, if $G$ is connected, then $\lambda\neq 0$. Moreover, when $p>1$, $U\varphi_p(x)$ is a vector and not a set. Therefore, $\langle \mathbf{1}, U\varphi_p(x)\rangle \in \langle \mathbf{1}, \triangle_p(x) \rangle/\lambda = \sum_{e} \vartheta_e \langle \mathbf{1},  \nabla f_e(x) \rangle f_e(x)^{p-1}/\lambda = 0$. This implies that $x$ contains both positive and negative components, which correspond to at least two weak (strong) nodal domains. Combining this result with that of Theorem~\ref{nodal} shows that the eigenvector corresponding to the eigenvalue $\lambda_2$ contains exactly two weak (strong) nodal domains.

For $p=1$, we only have $\langle \mathbf{1}, U\varphi_p(x)\rangle \ni \langle \mathbf{1}, \triangle_p(x) \rangle/\lambda=0$, which may allow that all components of $x$ are either nonnegative or nonpositive. An example of a graph with a single weak (strong) nodal domain may be found in Example 11 of~\cite{chang2017nodal}. 

\section{Proof for Lemma~\ref{medianproperty}}
According to the proof of Lemma~\ref{atleasttwo}, if $p>1$, we have $\langle \mathbf{1}, U\phi_p(x)\rangle = 0$ and thus $\mu_p^+(x) = \mu_p^-(x)$. Moreover, we have
 \begin{align}\label{diffc} 
\frac{\partial \|x - c\mathbf{1}\|_{\ell_p, \mu}^p}{\partial c}|_{c=0} = p \sum_{v\in [N]} \mu_v \sgn(x_v ) |x_v |^{(p-1)} = \mu_p^+(x) - \mu_p^-(x) = 0.
\end{align}
Hence, $c \in \arg\min_{c\in \mathbb{R}} \|x - c\mathbf{1}\|_{\ell_p, \mu}^p$.

If $p=1$, $0 \in \langle \mathbf{1}, U\phi_p(x)\rangle $, which implies $|\mu_1^+(x) - \mu_1^-(x)| \leq \mu^0(x)$. 
Furthermore, for any $c\geq 0$ we have
\begin{align*}
&\|x - c\mathbf{1}\|_{\ell_1, \mu} = \sum_{v: x_v > c} \mu_v(x_v - c) + \sum_{v:  0 \leq x_v < c} \mu_v(c- x_v) + \sum_{v:  x_v < 0} \mu_v(c- x_v) \\
= &\sum_{v: x_v > 0} \mu_v x_v - \sum_{v: x_v < 0} \mu_v x_v + 2\sum_{v:  0 < x_v < c} \mu_v(c- x_v) + c(\mu^0(x) +  \mu_1^-(x) -  \mu_1^+(x)) \\
\geq & \sum_{v: x_v > 0} \mu_v x_v - \sum_{v: x_v < 0} \mu_v x_v  = \|x\|_{\ell_1, \mu}.
\end{align*}
Therefore, $0\in \arg\min_{c\in \mathbb{R}} \|x - c\mathbf{1}\|_{\ell_1, \mu}$.

\section{Proof of Theorem~\ref{cheeger}}
Let us first prove the second part of the theorem. Suppose that $\{S_1^*, S_2^*, ..., S_k^*\}\in P_k$ is one $k$-way partition such that $h_k = \max_{i\in[k]} c(S_i^*)$. Let $\mathcal{A}=\text{Span}(\mathbf{1}_{S_1^*},\mathbf{1}_{S_2^*},...,\mathbf{1}_{S_k^*})$. Choose a vector $x\in \mathcal{A}\cap \mathcal{S}_{p,\mu}$ and suppose that it can be written as $x=\sum_{i\in[k]}\alpha_i \mathbf{1}_{S_i^*}$.
\begin{lemma}~\label{normalization}
If $x=\sum_{i\in[k]}\alpha_i \mathbf{1}_{S_i^*}$ and $x\in  \mathcal{S}_{p,\mu}$, then 
\begin{align*}
\sum_{i\in [k]} |\alpha_i|^p \text{vol}(S_i^*) = 1. 
\end{align*}
\end{lemma}
 \begin{proof}
 As $S_{i}^* \cap S_{j}^* = \emptyset$,  we have $1 = \|x\|_{\ell_p, \mu}^p  = \sum_{i\in [k]} \|\alpha_i\mathbf{1}_{S_{i}^*}\|_{\ell_p, \mu}^p = \sum_{i\in [k]} |\alpha_i|^p \text{vol}(S_i^*). $
 \end{proof}
\subsection{Arbitrary Submodular Weights}
First, consider the following chain of inequalities that leads to an upper bound for $\tilde{Q}_p(x)$:
\begin{align}
\tilde{Q}_p(x) & = \sum_{e} \vartheta_e (f_e(x))^p = \sum_{e} \vartheta_e  \langle \nabla f_e(x), x\rangle ^p = \sum_{e} \vartheta_e  \left[\sum_{i\in [k]}\alpha_i \langle \nabla f_e(x), \mathbf{1}_{S_i^*}\rangle \right]^p \nonumber\\
&\stackrel{1)}{\leq} \sum_{e} \vartheta_e \left(\min\{|e|, k\}\right)^{p-1} \sum_{i\in[k]}|\alpha_i|^p|\langle  \nabla f_e(x), \mathbf{1}_{S_i^*}\rangle|^p  \label{cheegercare1}\\
&\stackrel{2)}{\leq}  \left(\min\{\max |e|, k\}\right)^{p-1} \sum_{i\in[k]} |\alpha_i|^p\sum_e \vartheta_e (f_e(\mathbf{1}_{S_i^*}))^p\nonumber \\
&\stackrel{3)}{\leq} \left(\min\{\max |e|, k\}\right)^{p-1} \sum_{i\in[k]} |\alpha_i|^p \text{vol}(\partial S_i^*) \nonumber\\
& \stackrel{4)}{\leq} \left(\min\{\max |e|, k\}\right)^{p-1} \frac{\sum_{i\in[k]} |\alpha_i|^p \text{vol}(\partial S_i^*)}{\sum_{i\in [k]} |\alpha_i|^p \text{vol}(S_i^*)} \nonumber \\
& \leq \left(\min\{\max |e|, k\}\right)^{p-1}h_k. \nonumber
\end{align}
Here, 1) follows from $|\{i\in [k]|\langle  \nabla f_e(x), \mathbf{1}_{S_i^*\cap e}\rangle>0\}|\leq \min\{|e|, k\}$ and H\"{o}lder's inequality; 2) follows from the definition of $f_e$; 3) is a consequence of the inequality $\sum_{e}\vartheta_e (f_e(\mathbf{1}_{S_i^*}))^p \leq \sum_e \vartheta_e w_e(S_i)^p \leq  \sum_e \vartheta_e w_e(S_i)=\text{vol}(\partial S_i^*)$; and 4) follows from Lemma~\ref{normalization}. 

Before establishing the lower bound, we first prove the following lemma.
\begin{lemma}\label{cheegeracc}
For any vector $x \in \mathbb{R}_{\geq 0}^N \backslash \{\mathbf{0}\}$ and $p\geq 1$, there exists some $\theta \geq 0$ such that $\Theta(x,\theta) = \{u: x(u)>\theta\}$ satisfies
\begin{align*}
R_p(x) \geq \left(\frac{1}{\tau}\right)^{p-1} \left(\frac{c(S)}{p}\right)^p,
\end{align*}
where $\tau=\max_{v\in [N]}\frac{d_v}{\mu_v}$.
\end{lemma}
\begin{proof}
Let us consider the case $p>1$ first. For a vector $x$, we use $(x)^p$ to denote the coordinatewise $p$-th power operation. Furthermore, let $q = \frac{p}{p-1}$.  

For a vector $x'\in \mathbb{R}^N$, we write the Lov{\'a}sz extension $f_e(x')$ by only including arguments that lie in $e$, i.e.,  
\begin{align*}
f_e(x') = \sum_{k=1}^{|e|-1} w_e(S^{k,e})(x_{i_k(e)}'-x_{i_{k+1}(e)}')
\end{align*}
where $e = \{i_k(e)\}_{1\leq k\leq |e|}$, $x_{i_1(e)}'\geq x_{i_2(e)}' \geq \cdots \geq x_{i_{|e|}(e)}'$ and $S^{k,e} = \{i_j(e)\}_{1\leq j \leq k}$. Then,
\begin{align}
& \, Q_1(x^p) = \sum_{e} \vartheta_ef_e(x^p) = \sum_{e}\vartheta_e \sum_{k=1}^{|e|-1} w_e(S^{k,e})(x_{i_k(e)}^p-x_{i_{k+1}(e)}^p) \nonumber\\
\stackrel{1)}{\leq} &  \sum_{e}\vartheta_e \sum_{k=1}^{|e|-1} w_e(S^{k,e})p(x_{i_k(e)} -x_{i_{k+1}(e)}) \left(x_{i_k(e)}\right)^{p -1}  \label{cheegercare2}\\
=& \, p\sum_{e} \sum_{k=1}^{|e|-1} \vartheta_e^{1/p} w_e(S^{k,e})(x_{i_k(e)}-x_{i_{k+1}(e)})\, \vartheta_e^{1/q}\left(x_{i_k(e)}\right)^{p/q} \nonumber\\
\stackrel{2)}{\leq} & \, p \left\{\sum_{e}\sum_{k=1}^{|e|-1} \vartheta_e\left[w_e(S^{k,e})(x_{i_k(e)}-x_{i_{k+1}(e)})\right]^{p} \right\}^{\frac{1}{p}}\left\{\sum_{e}\vartheta_e\sum_{k=1}^{|e|-1}\left(x_{i_k(e)}\right)^{p} \right\}^{\frac{1}{q}} \nonumber\\
\leq& \, p (Q_p(x))^{\frac{1}{p}}\left(\sum_vd_vx_v^{p} \right)^{\frac{1}{q}}\nonumber \\ 
\leq & \,  p\tau^{1-\frac{1}{p}} \left(Q_{p}(x)\right)^{\frac{1}{p}}\|x\|_{\ell_p,\mu}^{p-1}, \nonumber
\end{align} 
where 1) follows from the fact that $a \geq b \geq 0$ implies $a^p - b^p \leq p(a - b)a^{p-1}$ and 2) is a consequence of H\"{o}lder's inequality. 
As when $p=1$, we naturally have $Q_1(x)\leq Q_1(x)$. For any $p\geq 1$, we have
\begin{align}
\frac{Q_1(x^p)}{\|x\|_{\ell_p,\mu}^{p}}\leq \, p\tau^{1-\frac{1}{p}} \frac{\left(Q_{p}(x)\right)^{\frac{1}{p}}}{\|x\|_{\ell_p,\mu}}. \label{cheegercare3}
\end{align}
Moreover, by representing Lov$\acute{\text{a}}$sz extension by its integral form~\cite{bach2013learning}, we obtain
\begin{align*}
Q_1(x^p) = \sum_e \vartheta_e\int_{0}^{+\infty} w_e(\{v: x_v^p> \theta\}\cap e) \, d\theta = \int_{0}^{+\infty} \vartheta_e\sum_e w_e(\{v: x_v^p> \theta\}\cap e) d\theta.
\end{align*} 
Then,
\begin{align*}
\frac{Q_1(x^p)}{\|x\|_{\ell_p,\mu}^{p}} &=\frac{ \int_{0}^{+\infty}\vartheta_e \sum_e w_e(\{v: x_v^p> \theta\}\cap e) d\theta}{ \int_{0}^{+\infty} \mu(\{v: x_v^p> \theta\}) d\theta}  \geq \inf_{\theta\geq 0} \frac{\sum_e \vartheta_ew_e(\{v: x_v^p> \theta\}\cap e)}{  \mu(\{v: x_v^p> \theta\})} \\
&=  \inf_{\theta\geq 0} \frac{\text{vol}(\partial \{v: x_v^p> \theta\})}{  \text{vol}( \{v: x_v^p> \theta\})} = \inf_{\theta\geq 0} c(\{v: x_v^p>\theta\})
\end{align*} 
Therefore, the minimizer $\theta^*$ induces a set $\Theta^*=\{v: x_v^p> \theta^*\}\subseteq A$, for which the following inequality holds
\begin{align*}
R_p(x) = \frac{Q_{p}(x)}{\|x\|_{\ell_p,\mu}^p} \geq \left(\frac{Q_1(x^p)}{\|x\|_{\ell_p,\mu}^{p}} \right)^p \frac{1}{p^p \tau^{p-1}} = \left(\frac{1}{\tau}\right)^{p-1} \left(\frac{c(\Theta^*)}{p}\right)^p.
\end{align*}
This proves Lemma~\ref{cheegeracc}. 

Next, we turn our attention to the first inequality of Theorem~\ref{cheeger}. Suppose $\lambda_k^{(p)}$ has a corresponding eigenvector $x$ that induces the strong nodal domains $A_1, A_2, ..., A_m$. According to Lemma~\ref{nodalanalysis}, we know that $\lambda_k^{(p)}\geq R_p(\mathbf{1}_{A_i})$. Moreover, due to Lemma~\ref{cheegeracc}, for any $i\in[m]$, there exists a $B_i\subseteq A_i$ such that
\begin{align*}
R_p(\mathbf{1}_{A_i}) \geq  \left(\frac{1}{\tau}\right)^{p-1} \left(\frac{c(B_i)}{p}\right)^p.
\end{align*}
Therefore,
\begin{align*}
&\lambda_k^{(p)}\geq \max_{i\in[m]}R_p(\mathbf{1}_{A_i}) \geq \max_{i\in[m]}\left(\frac{1}{\tau}\right)^{p-1} \left(\frac{c(B_i)}{p}\right)^p \\
&\geq \min_{(B_1,B_2,...,B_m)\in P_m}\max_{i\in[m]}\left(\frac{1}{\tau}\right)^{p-1} \left(\frac{c(B_i)}{p}\right)^p  \geq \left(\frac{1}{\tau}\right)^{p-1} \left(\frac{h_m}{p}\right)^p. 
\end{align*}
\end{proof}

\subsection{Homogeneous Weights}
We can use a similar approach to prove the previous result for homogeneous weights, i.e., weights such that $w_e(S)=1$ for all $S\in 2^e\backslash \{\emptyset, e\}$. Only several steps have to be changed. 

First, the inequality~\eqref{cheegercare1} may be tightened. Again, consider the partition $\{S_1^*, S_2^*, ..., S_k^*\}\in P_k$ such that $h_k = \max_{i\in[k]} c(S_i^*)$. For a given hyperedge $e$, choose a pair of vertices $(u^*,v^*) \in \arg\max_{u,v\in e} |x_{u} - x_{v}|^p $. If both $u, v \in  S_i^*$, then $f_e(x) = 0$. If not, assume that $u\in S_i^*$ and $v\in S_j^*$. Then,
\begin{align*}
(f_e(x))^p & =  |x_{u^*} - x_{v^*}|^p \leq 2^{p-1}(|x_{u^*}|^p + |x_{v^*}|^p) \\
 &\leq 2^{p-1} (|\alpha_i|^pf_e(1_{S_i^*})^p +  |\alpha_j|^pf_e(1_{S_j^*})^p) = 2^{p-1}\sum_{i\in[k]} |\alpha_i|^pf_e(\mathbf{1}_{S_i^*})^p. 
\end{align*}
Therefore, in the homogeneous case, we have  
\begin{align*}
R_p(x) \leq 2^{p-1}h_k.
\end{align*}
Second, we will use the following lemma to prove the lower bound:
\begin{lemma}[\cite{amghibech2003eigenvalues}]
If $a, b \geq 0$, $p > 1$, then 
\begin{align*}
a^p - b^p \leq \frac{p}{2^{1-\frac{1}{p}}}(a-b) \left(a^p+b^p\right)^{1-\frac{1}{p}}
\end{align*}
\end{lemma}
So the inequality~\eqref{cheegercare2} may be tightened as
\begin{align*}
 f_e(x^p) = x_{i_0(e)}^p-x_{i_{|e|}(e)}^p \leq  \frac{p}{2^{1-\frac{1}{p}}}(x_{i_0(e)} -x_{i_{|e|}(e)}) \left(x_{i_0(e)}^p+x_{i_{|e|}(e)}^p\right)^{1-\frac{1}{p}}.
\end{align*}
With these two modifications, we can rewrite inequality~\eqref{cheegercare3} as
\begin{align*}
\frac{Q_1(x^p)}{\|x\|_{\ell_p,\mu}^{p}}\leq \frac{p}{2^{1-\frac{1}{p}}}\tau^{1-\frac{1}{p}} \frac{\left(Q_{p}(x)\right)^{\frac{1}{p}}}{\|x\|_{\ell_p,\mu}},
\end{align*}
which leads to
\begin{align*}
\lambda_k^{(p)} \geq \left(\frac{2}{\tau}\right)^{p-1} \left(\frac{h_m}{p}\right)^p.
\end{align*}

\section{Proof of Theorem~\ref{algoform}}

First, we prove that $\lambda_2^{(p)} \geq \inf_x \mathcal{R}_p(x)$. Suppose that $x'$ is a nonconstant eigenvector corresponding to $\lambda_2^{(p)}$. If $\lambda_2^{(p)}=0$. If $\lambda_2^{(p)}=0$, then $\langle x', \triangle_p(x')\rangle = \langle x', \lambda_2^{(p)} U x'\rangle  = 0$, which implies that $Q_p(x') = 0$. Moreover, as $x'$ is nonconstant,  $\min_{c\in \mathbb{R}}\|x' - c\mathbf{1}\|_{\ell_p,\mu}^p>0$, and thus $\mathcal{R}_p(x') = 0 \leq  \lambda_2^{(p)}$. This proves the claim of the theorem for the case that $\lambda_2^{(p)}=0$. Next, suppose that $\lambda_2^{(p)} \not = 0$. First, we observe that Lemma~\ref{medianproperty} implies $0 \in \nabla_c Z_{p,\mu}(x', c)|_{c=0} $. As $Z_{p,\mu}(x', c)$ is convex in $c$, $c=0$ is a minimizer of $Z_{p,\mu}(x', c)$, $i.e., Z_{p,\mu}(x', 0) = Z_{p,\mu}(x')$. Moreover, $\lambda_2^{(p)} = R_p(x') = \frac{Q_p(x')}{Z_{p,\mu}(x', 0)} = \frac{Q_p(x')}{Z_{p,\mu}(x')} =\mathcal{R}_p(x')$. Therefore, $\lambda_2^{(p)}\geq \inf_x \mathcal{R}_p(x)$. 

Second, we prove that $ \inf_x \mathcal{R}_p(x) \geq \lambda_2^{(p)}$ . First, we focus on the case $p>1$. For any $t_1\in\mathbb{R}\backslash \{0\}$ and $t_2 \in\mathbb{R}$, it is easy to show that $\mathcal{R}_p(t_1 x + t_2\mathbf{1}) = \mathcal{R}_p(x)$. Therefore, to characterize the infimum of $\mathcal{R}_p(x)$, it suffices to consider $x\in \mathcal{S}_{p,\mu}\cap \mathcal{A}$, where $\mathcal{A} =  \{x\in \mathbb{R}^N| 0\in \arg\min_c Z_{p,\mu}(x,c)\}$. For $p>1$, $Z_{p,\mu}(x,c)$ is differentiably convex in $c$. By using formula~\eqref{diffc} once again, we know that $\mathcal{A}=\{x\in \mathbb{R}^N| \mu_p^+(x) - \mu_p^-(x) =0\}$. Furthermore, $\mathcal{A}$ is closed, since the functions $\mu_p^+, \mu_p^-$ are continuous. By recalling that $\mathcal{S}_{p,\mu}$ is a compact space we know that there exists a point $x_*\in \mathcal{S}_{p,\mu}\cap \mathcal{A}$ such that $x_* \in \arg\inf_x \mathcal{R}_p(x)$. 

Consider next the subspace $\mathcal{A}' = \{t_1x_*+ t_2 \mathbf{1}: t_1, t_2\in \mathbb{R}\}$. As $x_*$ being nonconstant reduces to $x_*\neq c\textbf{1}$ for any scalar $c\in\mathbb{R}$, we have $\gamma(\mathcal{A}\cap\mathcal{S}_p) = 2$. According to the definition of $\lambda_2^{(p)}$~\eqref{defcrit}, it follows that 
\begin{align*}
\lambda_2^{(p)}& \leq \max_{x\in \mathcal{A}' \cap \mathcal{S}_{p,\mu}} Q_p(x) =  \max_{t_1, t_2\in \mathbb{R}} Q_p(\frac{t_1x_* + t_2 \mathbf{1}}{\|t_1x_* + t_2 \mathbf{1}\|_{\ell_p,\mu}}) =  \max_{t_1, t_2\in \mathbb{R}} \frac{Q_p(t_1x_*) }{\|t_1x_* + t_2 \mathbf{1}\|_{\ell_p,\mu}^p}  \\
&= \max_{t_1\in \mathbb{R}}  \frac{Q_p(t_1x_*)}{Z_{p,\mu}(t_1x_*)} = \mathcal{R}_{p}(x_*).
\end{align*}
For any $a,b\in \mathbb{R}$, we can write $Q_p(ax + b\mathbf{1}) = |a|^pQ_p(x)$ and  $Z_{p,\mu}(ax + b\mathbf{1}) = |a|^pZ_{p,\mu}(x)$. Combining these expressions with $\lambda_2^{(p)} \geq \inf_x \mathcal{R}_p(x)$ shows that $\lambda_2^{(p)}= \inf_x \mathcal{R}_p(x)$. This settles the case $p>1$.

Next, we turn our attention to proving that $ \min_x \mathcal{R}_1(x) = h_2$ for $p=1$. This result, combined with the inequality $h_2\geq \lambda_2^{(1)}$ from Theorem~\ref{cheeger} proves that $ \inf_x \mathcal{R}_1(x) = h_2 =\lambda_2^{(1)}$. 

Recall that the $2$-way Cheeger constant can be written as $ \min_{S \subset [N]}\frac{|\partial S|}{\min\{\text{vol}(S), \text{vol}([N]\backslash S)\}}$.
This expression, along with the fact that $\inf_x \mathcal{R}_1(x) = h_2$ (which is a special case of Theorem 1 in~\cite{buhler2013constrained}), allows one to reduce the proof to showing that the Lov$\acute{\text{a}}$sz extensions of $\text{vol}(\partial S)$ and $\min\{\text{vol}(S), \text{vol}([N]\backslash S)\}$ are equal to $Q_1(x)$ and $Z_{1,\mu}(x)$, respectively. The claim regarding $Q_1$ naturally follows from the Definition~\ref{lapdef}. We hence only need to show that the Lov$\acute{\text{a}}$sz extension of $\min\{\text{vol}(S), \text{vol}([N]\backslash S)\}$ equals $Z_{1,\mu}(x)$.

For a given $x\in \mathbb{R}^N$, suppose that $x_{i_1} \geq x_{i_2} \geq \cdots \geq x_{i_N} $. Then, the Lov$\acute{\text{a}}$sz extension of $\min\{\text{vol}(S), \text{vol}([N]\backslash S)\}$ can be written as
\begin{align}\label{balanceset}
\sum_{k = 1}^{N} \min\{{\sum_{j=1}^k \mu_{i_j}, \sum_{j=k+1}^N \mu_{i_j} \}}\, (x_{i_j} - x_{i_{j+1}}).
\end{align}
Let $k^*$ be equal to $\min\left\{k\in\{1, 2,...,N\}: \sum_{j=1}^k \mu_{i_j} \geq \sum_{j=k+1}^N \mu_{i_j}\right\}$. In this case, \eqref{balanceset} is equivalent to 
\begin{align*}
\sum_{k = 1}^{k^*-1} \mu_{i_k}(x_{i_k} - x_{i_{k^*}}) + \sum_{k = k^*+1}^{N} \mu_{i_k}(x_{i_{k^*}} - x_{i_{k}})  = \|x - x_{i_{k^*}} \mathbf{1}\|_{\ell_1, \mu} = Z_{1,\mu}(x), 
\end{align*}
which establishes the claimed result. 

\section{Proof for Theorem~\ref{thresholding}}
For a vector $x \in \mathbb{R}^N$, define two vector $x^+, x^- \in \mathbb{R}^N$ according to $(x^+)_v = \max \{x_v, 0\}$ and $(x^-)_v = \max \{-x_v, 0\}$. Hence, $x = x^+ - x^-$ and $x^+, -x^- \rightharpoonup  x$. Then,
\begin{align*}
Q_p(x) &= \sum_{e} \vartheta_e f_e(x)^p = \sum_{e} \vartheta_e [\langle \nabla f_e(x), x^+ \rangle + \langle \nabla f_e(x), -x^- \rangle ]^p \stackrel{1)}{=} \sum_{e}  \vartheta_e  [f_e(x^+) + f_e(-x^-)]^p \\
&\stackrel{2)}{=}  \sum_{e}  \vartheta_e  [f_e(x^+)^p + f_e(x^-)^p] = Q_p(x^+) + Q_p(x^-),
\end{align*}
where in $1)$ we used Lemma~\ref{dotproductpre}, and in $2)$ we used the fact that $f_e(x) = f_e(-x)$ and $(a+b)^p \geq a^p + b^p$ for $a,b \geq 0,\, p\geq 1$. 
Moreover, as $Z_{p,\mu}(x) = \|x\|_{\ell_p, \mu}^p = \|x^+\|_{\ell_p, \mu}^p + \|x^-\|_{\ell_p, \mu}^p$, we have 
\begin{align*}
\mathcal{R}_p(x) \geq \min\{R_p(x^+), R_p(x^-)\}.
\end{align*} 
By applying Lemma~\ref{cheegeracc} on $x^+$ and $x^-$, and by observing that $c(x^+), c(x^-) \geq c(x)$, we have 
\begin{align*}
\mathcal{R}_p(x) \geq \min\{R_p(x^+), R_p(x^-)\} \geq \left(\frac{1}{\tau}\right)^{p-1} \left(\frac{\min\{c(x^+), c(x^-)\}}{p}\right)^p \geq \left(\frac{1}{\tau}\right)^{p-1} \left(\frac{c(x)}{p}\right)^p,
\end{align*} 
which concludes the proof. 

\section{Proof for Lemma~\ref{SDPapprox}}
First, it can be easily shown that $Ux\perp \mathbf{1}$, since
\begin{align*}
\sum_{v\in [N]} \mu_v x_v =\sum_{v\in [N]} \mu_v (x'_v)^T g = (\sum_{v\in [N]} \mu_v x'_v)^T g= 0. 
\end{align*}
Next, we establish a lower bound for $\|x\|_{\ell_2, \mu}^2$. For this purpose, we find the following lemma useful. 
\begin{lemma}[Lemma 7.7~\cite{louis2015hypergraph}]~\label{boundnorm}
Let $Y_1$, $Y_2$, ..., $Y_k$ be zero-mean normal random variables that are not necessarily independent, such that $\mathbb{E}[\sum_{i} Y_i^2]= 1.$ Then, 
\begin{align*}
\mathbb{P}\left[\sum_{i} Y_i^2 \geq \frac{1}{2}\right] \geq \frac{1}{12}. 
\end{align*}
\end{lemma}
We start by observing that
\begin{align*}
\mathbb{E}[\|x\|_{\ell_2, \mu}^2] = \mathbb{E}[\|X^Tg\|_{\ell_2, \mu}^2]  = \sum_{v\in [N]} \mu_v \|x'_v\|_2^2  = 1.
\end{align*}
From Lemma~\ref{boundnorm}, it follows that
\begin{align}\label{demoninator}
\mathbb{P}\left[\|x\|_{\ell_2, \mu}^2 \geq \frac{1}{2}\right]\geq \frac{1}{12}.
\end{align}
Next, we prove an upper bound for $Q_2(x)$. For any $e\in E$, $w\in \mathcal{E}(\mathcal{B}_e)$, we have
\begin{align}
\mathbb{E}\left[\left(\max_{y\in \mathcal{E}(\mathcal{B}_e)}\langle y, x' \rangle\right)^2\right] &= \mathbb{E}\left[\left(\max_{y\in \mathcal{E}(\mathcal{B}_e)}\left\langle g, \frac{Xy}{\|Xy\|_2} \right\rangle\right)^2 \|Xy\|_2^2 \right] \nonumber \\
&\leq \mathbb{E}\left[\left(\max_{y\in \mathcal{E}(\mathcal{B}_e)}\left\langle g, \frac{Xy}{\|Xy\|_2} \right\rangle\right)^2\right] \max_{y'\in \mathcal{E}(\mathcal{B}_e)} \|Xy'\|_2^2. \label{firststep}
\end{align}
Suppose that the hyperedge $e$ contains the following vertices $e = \{v_1, v_2, \dots, v_{|e|}\}$. Let $\mathbb{A} = \text{Span}(x'_{v_1} - x'_{v_{|e|}}, x'_{v_2} -x'_{v_{|e|}}, ...,  x'_{v_{|e|-1}} - x'_{v_{|e|}})$ and let $
\mathbb{S}^n$ stand for the unit ball in $\mathbb{R}^n$. Recall $n$ is the dimension of the space to embed the vectors for SDP relaxation which is no less than  $\zeta(E)$. Then, given that $\sum_{v\in e} y_v= 0$ and $y_u= 0$ for $u\notin e$, $\frac{Xy}{\|Xy\|_2}$ always lies in $\mathbb{A}\cap \mathbb{S}^n$. Therefore,
\begin{align}\label{secondstep}
\mathbb{E}\left[\left(\max_{y\in \mathcal{E}(\mathcal{B}_e)}\left\langle g, \frac{Xy}{\|Xy\|_2} \right\rangle \right)^2 \right] \leq \mathbb{E}\left[\left(\max_{x' \in \mathbb{A}\cap \mathbb{S}^n}\left\langle g, x'  \right\rangle \right)^2 \right] = \text{dim}(\mathbb{A}) = |e| -1. 
\end{align}
Combining \eqref{firststep} with \eqref{secondstep}, we have
\begin{align}\nonumber
\mathbb{E}[(\max_{y\in \mathcal{E}(\mathcal{B}_e)}\langle y, x' \rangle)^2] \leq (|e| -1) \max_{y\in \mathcal{E}(\mathcal{B}_e)} \|Xy\|_2^2.
\end{align}
As $Q_2(x) = \sum_{e\in E}w_e(\max_{y\in \mathcal{E}(\mathcal{B}_e)}\langle y, x \rangle)^2$, using Markov's inequality, we have 
\begin{align}\label{numerator}
\mathbb{P}\left( Q_2(x)  \geq 13 \, \zeta(E) \, \sum_{e\in E} \max_{w'\in \mathcal{E}(\mathcal{B}_e)} \|Yw'\|_2^2 \right) \leq \frac{1}{13}.
\end{align}
In addition, applying the union bound to \eqref{numerator} and using \eqref{demoninator}, we have 
\begin{align}
\mathbb{P}\left( \mathcal{R}_2(x) \leq 26\;\text{SDPopt} \right) \geq \frac{1}{13}.
\end{align}
which concludes the proof. 

Note that the distortion term $O(\zeta(E))$ is introduced through the inequalities~\eqref{firststep} and \eqref{secondstep}, which are tight for this case. This may be shown as follows. Suppose the solution of the SDP produces a collection of vectors $\{x'_{v_i}\}_{1\leq i\leq|e|}$ that have the same $\ell_2-$norm, i.e. $\|x'_{v_i}\|_2 = a$,  and are orthogonal in $\mathbb{R}^{n}$. Let $\mathcal{B}_e$ denote the base polytope corresponding to a submodular function satisfying $w_e(S) = \frac{2}{|e|}\min\{|S\cap e|, |e| -|S\cap e|\}$. Define a subset of $\mathcal{B}_e$, $\mathcal{B}_{e, \text{s}}$, as follows
\begin{align*}
\mathcal{B}_{e, \text{s}} \triangleq\left\{y\in\mathbb{R}^N| |y(\{v_i\})| \leq \frac{2}{|e|}, y(\{v_{i+|e|/2}\}) = - y(\{v_i\}),\;\text{for $1\leq i\leq |e|/2$,}\; , y(\{v\}) = 0, \;\text{for $v \notin e$} \right \}.
\end{align*}
Then, choosing a $y'$ in $B_{e, \text{s}}$ such that $y'(\{v_i\}) =\frac{2}{|e|} \frac{\langle g, x'_{v_i} - x'_{v_{i+|e|/2}}\rangle}{|\langle g, x'_{v_i} - x'_{v_{i+|e|/2}}\rangle|}$ for $1\leq i\leq|e|/2$, we obtain 
\begin{align*}
\mathbb{E}\left[\left(\max_{y\in \mathcal{E}(\mathcal{B}_e)}\left\langle g, Xy \right\rangle\right)^2\right] & \geq \mathbb{E}\left[\left\langle g, Xy' \right\rangle^2\right]\\
&=\mathbb{E}\left[\left\langle g, \sum_{1\leq i \leq |e|/2} \frac{2}{|e|} \frac{\langle g, x'_{v_i} - x'_{v_{i+|e|/2}}\rangle}{|\langle g, x'_{v_i} - x'_{v_{i+|e|/2}}\rangle|}(x'_{v_i} - x'_{v_{i+|e|/2}}) \right\rangle^2\right]  \\
&\geq \frac{4}{|e|^2} \mathbb{E}\left[\sum_{1\leq i \leq |e|/2} |\langle g, x'_{v_i} - x'_{v_{i+|e|/2}}\rangle|\right]^2\\
&\geq \frac{4}{|e|^2} \left[\sum_{1\leq i\leq |e|/2} \mathbb{E}|\langle g, x'_{v_i} - x'_{v_{i+ |e|/2}}\rangle|\right]^2 \\
& = \frac{4}{|e|^2} \left(\frac{|e|}{2} \times \sqrt{2}a \sqrt{\frac{2}{\pi}}\right)^2=  \frac{|e|}{\pi} \frac{4}{|e|}a^2  \\
&\geq \frac{|e|}{\pi}a^2\max_{y\in \mathcal{E}(\mathcal{B}_e)} \|y\|^2 \geq \frac{|e|}{\pi}\max_{y\in \mathcal{E}(\mathcal{B}_e)} \sum_{1\leq i\leq |e|} \|y(\{v_i\}) x_{v_i}\|_2^2 \\
& = \frac{1}{\pi}|e|\max_{y\in \mathcal{E}(\mathcal{B}_e)} \|Xy\|^2,
\end{align*}
where the last equality is using the assumption that $\{x_{v_i}\}_{v_i \in e}$ are mutually orthogonal. Therefore, the Gaussian projection $X$ causes distortion $\Theta(|e|)$. 

\section{Proof of Theorem~\ref{SDPfinal}}

By combining Theorem~\ref{algoform}, Theorem~\ref{thresholding}, Lemma~\ref{SDPapprox} and Theorem~\eqref{cheeger}, we obtain 
\begin{align*}
c(x) &\leq O(\sqrt{\tau}) \mathcal{R}_2(x)^{1/2} \leq  O(\sqrt{\zeta(E)\, \tau}) \left(\inf_{x} \mathcal{R}_2(x)\right)^{1/2} \\
&= O(\sqrt{\zeta(E) \, \tau}) \left(\lambda_2^{(2)}\right)^{1/2} \leq O(\sqrt{\zeta(E)\, \tau h_2}) \,\text{w.h.p.}
\end{align*}

\section{Proof of Theorem~\ref{convergence}}
 First, according to Step 3, we have 
 \begin{align*}
 Q_1(z^{k+1}) - \hat{\lambda}^k \langle z^{k+1}, g^k\rangle \leq Q_1(z^{k}) -  \hat{\lambda}^k \langle z^{k}, g^k\rangle.
 \end{align*}
 It is also straightforward to check that $g^{k}$ satisfies 
 \begin{align*}
 \langle \mathbf{1}, g^k \rangle = 0,\quad \quad \langle x^{k}, g^k\rangle = \|x^{k}\|_{\ell_1,\mu}.
 \end{align*}
 Therefore,
  \begin{align*}
 Q_1(x^{k+1}) -  \hat{\lambda}^k \langle x^{k+1}, g^k\rangle &= Q_1(z^{k+1}) -\hat{\lambda}^k \langle z^{k+1}, g^k\rangle \leq Q_1(z^{k}) - \hat{\lambda}^k \langle z^{k}, g^k\rangle \\
 &= Q_1(x^{k}) - \hat{\lambda}^k \langle x^{k}, g^k\rangle  = Q_1(x^{k}) - \hat{\lambda}^k \|x^{k}\|_{\ell_1,\mu} = 0,
 \end{align*}
 which implies 
  \begin{align*}
\mathcal{R}_1(x^{k+1}) \leq \hat{\lambda}^k \frac{\langle x^{k+1}, g^k\rangle}{Z_{1,\mu}(x^{k+1})} = \hat{\lambda}^k \frac{\langle x^{k+1}, g^k\rangle}{\|x^{k+1}\|_{\ell_1,\mu}} \leq \hat{\lambda}^k \frac{\|x^{k+1}\|_{\ell_1,\mu} \|g^{k}\|_{\ell_{\infty},\mu^{-1}}}{\|x^{k+1}\|_{\ell_1,\mu}} \stackrel{1)}{\leq} \hat{\lambda}^k. 
 \end{align*}
 Here, $1)$ follows from Lemma 3.11 which implies $\|g^{k}\|_{\ell_{\infty},\mu^{-1}}\leq 1$. This proves the claimed result. 

\section{Proof of Theorem~\ref{diffprob}}
If the norm $\|z\|$ stands for $\|z\|_2$, the duality result holds since
\begin{align*}
\min_{z:  \|z\|_2\leq 1} Q_1(z) - \hat{\lambda}^k \langle z, g^k\rangle &= \min_{z}\max_{\lambda\geq 0} \max_{y_e\in \vartheta_e\mathcal{B}_e} \sum_{e} \langle y_e, z\rangle -   \hat{\lambda}^k \langle z, g^k\rangle + \frac{\lambda}{2} (\|z\|_2^2 -1) \\
&=\max_{y_e\in \vartheta_e \mathcal{B}_e}\max_{\lambda\geq 0}\min_{z}\sum_{e} \langle y_e, z\rangle -  \hat{\lambda}^k \langle z, g^k\rangle + \frac{\lambda}{2} (\|z\|_2^2 -1) \\
& = \max_{y_e\in \vartheta_e\mathcal{B}_e} \max_{\lambda\geq 0} -\frac{ \|\sum_{e\in E} y_e - \hat{\lambda}^k g^k\|_2^2}{2\lambda} - \frac{\lambda}{2} \\
& = \max_{y_e\in \vartheta_e \mathcal{B}_e} -\|\sum_{e\in E} y_e - \hat{\lambda}^k g^k\|_2.
\end{align*}
The relationships between the primal and dual variables read as $z =\frac{ \hat{\lambda}^k g^k- \sum_{e\in E} y_e }{\lambda}$ and $\lambda =  \|\sum_{e\in E} y_e - \hat{\lambda}^k g^k\|_2$. 

If the norm $\|z\|$ stands for $\|z\|_{\infty}$, let $z' = (z + \mathbf{1})/2$. As $Q_1(z') = Q(z)/2$ and $\langle g^k, z' \rangle = \langle g^k, z \rangle/2$, we have 
\begin{align*}
\min_{z:  \|z\|_{\infty}\leq 1} \frac{1}{2}[Q_1(z) - \hat{\lambda}^k \langle z, g^k\rangle] \quad  \iff \quad  \min_{z': z' \leq [0, 1]^N}  Q_1(z') - \hat{\lambda}^k\langle z', g^k\rangle
\end{align*}
The right hand side essentially reduces to the following discrete optimization problem (Proposition 3.7~\cite{bach2013learning})
\begin{align*}
\min_{S\subseteq [N]} \sum_{e} \vartheta_e w_e(S) - \hat{\lambda}^k  g^k(S),
\end{align*}
where the primal and dual variables satisfy $z'_v = 1,$ if $v\in S,$ or $0$ if $v\not\in S$. 

\end{document}